%% file: root.tex
\documentclass{cls/IEEEojcsys}

\usepackage{amsmath,amssymb,amsfonts, dsfont} 		
\usepackage{mathtools} 
\usepackage{subfigure} 

\usepackage{color,array}
\usepackage{graphicx}
\usepackage[colorlinks,urlcolor=blue,linkcolor=blue,citecolor=blue]{hyperref}

\jvol{00}
\jnum{XX}
\paper{1234567}
\pubyear{2021}
\receiveddate{XX September 2021}
\accepteddate{XX October 2021}
\currentdate{XX November 2021}
\doiinfo{OJCSYS.2021.Doi Number}

\newtheorem{theorem}{Theorem}

\newtheorem{proposition}{Proposition} %
 
\newtheorem{definition}{Definition} 
\newtheorem{assumption}{Assumption}
\newtheorem{problem}{Problem}

\input{cls/sym.tex}

\input{cls/sym_zhl.tex}

\begin{document}

\sptitle{Article Category}

\title{Robust and Safe Autonomous Navigation for Systems with Learned SE(3) Hamiltonian Dynamics} 

\author{Zhichao Li (Student Member, IEEE)}

\author{Thai Duong (Student Member, IEEE)}

\author{Nikolay Atanasov (Member, IEEE)}

\corresp{CORRESPONDING AUTHOR: Zhichao Li (e-mail: \href{mailto:zhichaoli@ucsd.edu}{zhichaoli@ucsd.edu})}
\authornote{This work was supported by NSF RI IIS-2007141 and NSF CCF-2112665 (TILOS). }

\markboth{PREPARATION OF PAPERS FOR IEEE OPEN JOURNAL OF CONTROL SYSTEMS}{Zhichao Li {\itshape ET AL}.}

\begin{abstract}
Stability and safety are critical properties for successful deployment of automatic control systems. As a motivating example, consider autonomous mobile robot navigation in a complex environment. A control design that generalizes to different operational conditions requires a model of the system dynamics, robustness to modeling errors, and satisfaction of safety \NEWZL{constraints}, such as collision avoidance. This paper develops a neural ordinary differential equation network to learn the dynamics of a Hamiltonian system from trajectory data. The learned Hamiltonian model is used to synthesize an energy-shaping passivity-based controller and analyze its \emph{robustness} to uncertainty in the learned model and its \emph{safety} with respect to constraints imposed by the environment. Given a desired reference path for the system, we extend our design using a virtual reference governor to achieve tracking control. The governor state serves as a regulation point that moves along the reference path adaptively, balancing the system energy level, model uncertainty bounds, and distance to safety violation to guarantee robustness and safety. Our Hamiltonian dynamics learning and tracking control techniques are demonstrated on \Revised{simulated hexarotor and quadrotor robots} navigating in cluttered 3D environments.
\end{abstract}

\begin{IEEEkeywords}
	constrained control, physics-constrained learning, safe learning for control
\end{IEEEkeywords}
\maketitle

\input{tex/introduction.tex}

\input{tex/problem.tex}

\input{tex/tech_dyn_learning.tex}

\input{tex/tech_IDA_PBC_controller_v2.tex}

\input{tex/tech_IDA_PBC_robustness_v2.tex}


\input{tex/tech_safe_stab.tex}
\input{tex/robust_ref_gvn.tex}
\input{tex/extension_Rn_underactuated.tex}
\input{tex/evaluation.tex}

\input{tex/conclusion}

\section*{ACKNOWLEDGMENT}
We gratefully acknowledge support from NSF RI IIS-2007141 and NSF CCF-2112665 (TILOS).

\appendices
\input{tex/appendix_robust_se3.tex}

\bibliographystyle{cls/IEEEtran}
\bibliography{bib/thai_zhl.bib}

\begin{IEEEbiography}[{\includegraphics[width=2.5cm,height=3.2cm,clip,keepaspectratio]{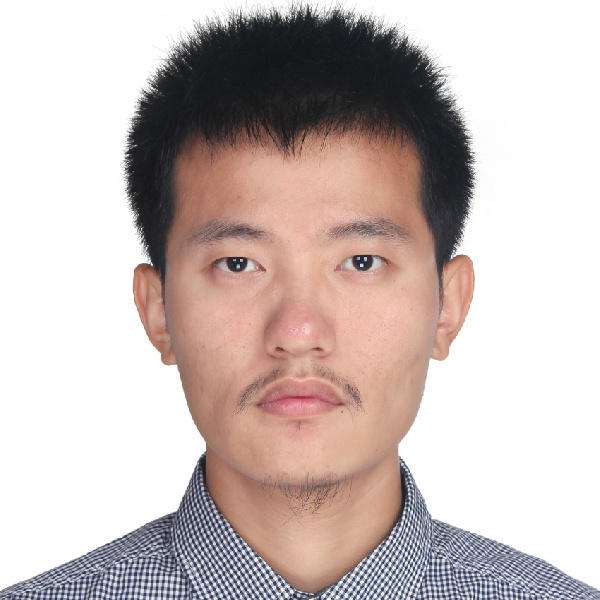}}]{Zhichao Li}~%
is a PhD student in Electrical and Computer Engineering at the University of California San Diego, La Jolla, CA, USA. He received a B.Eng. degree from the School of Astronautics, Northwestern Polytechnical University, Xi'an, Shaanxi, China in 2013 and an M.S. degree in Electrical Engineering from the School of Electrical, Computer and Energy Engineering, Arizona State University, Tempe, AZ, in 2016. His research interests include control theory and motion planning with applications to mobile robots. 
\end{IEEEbiography}

\begin{IEEEbiography}[{\includegraphics[width=2.5cm,height=3.2cm,clip,keepaspectratio]{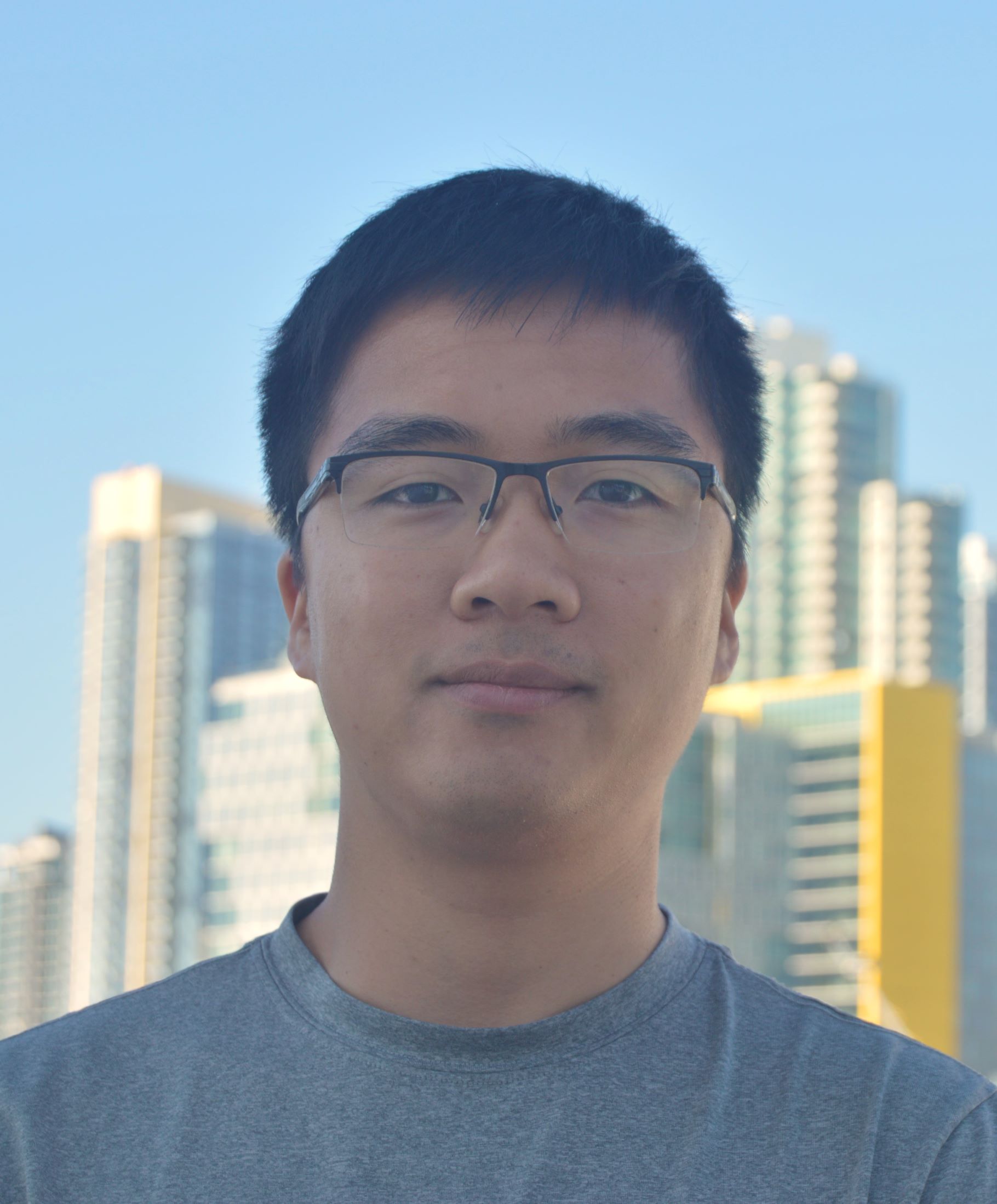}}]{Thai Duong}~%
	is a PhD student in Electrical and Computer Engineering at the University of California San Diego, CA, USA. He received a B.S. degree in Electronics and Telecommunications from Hanoi University of Science and Technology, Hanoi, Vietnam in 2011 and an M.S. degree in Electrical and Computer Engineering from Oregon State University, Corvallis, OR, in 2013. His research interests include machine learning with applications to robotics, mapping and active exploration using mobile robots, robot dynamics learning, and decision making under uncertainty.
\end{IEEEbiography}

\begin{IEEEbiography}[{\includegraphics[width=2.5cm,height=3.2cm,clip,keepaspectratio]{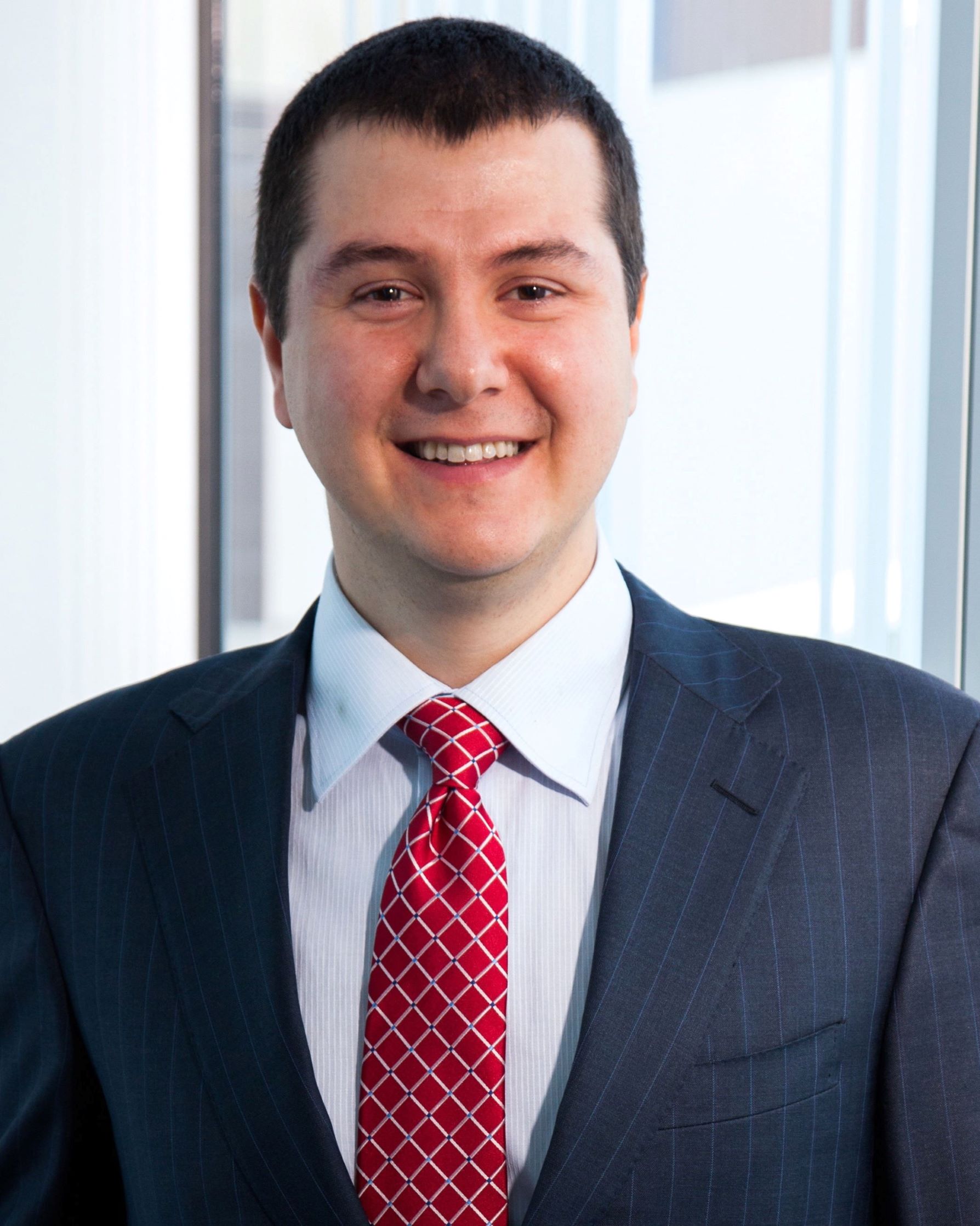}}]{Nikolay Atanasov}~%
	is an Assistant Professor of Electrical and Computer Engineering at the University of California San Diego, La Jolla, CA, USA. He obtained a B.S. degree in Electrical Engineering from Trinity College, Hartford, CT, USA in 2008 and M.S. and Ph.D. degrees in Electrical and Systems Engineering from the University of Pennsylvania, Philadelphia, PA, USA in 2012 and 2015, respectively. His research focuses on robotics, control theory, and machine learning, applied to active perception problems for ground and aerial robots. He works on probabilistic models that unify geometry and semantics in simultaneous localization and mapping (SLAM), as well as optimal control and reinforcement learning techniques for minimizing uncertainty in these models. Dr. Atanasov's work has been recognized by the Joseph and Rosaline Wolf award for the best Ph.D. dissertation in Electrical and Systems Engineering at the University of Pennsylvania in 2015, the best conference paper award at the 2017 IEEE International Conference on Robotics and Automation (ICRA), and a 2021 NSF CAREER award. 
\end{IEEEbiography}

\end{document}

%% file: cls/sym.tex
\newcommand{\calA}{{\cal A}}

\newcommand{\calD}{{\cal D}}

\newcommand{\calF}{{\cal F}}

\newcommand{\calH}{{\cal H}}

\newcommand{\calL}{{\cal L}}

\newcommand{\calO}{{\cal O}}
\newcommand{\calP}{{\cal P}}

\newcommand{\calS}{{\cal S}}
\newcommand{\calT}{{\cal T}}
\newcommand{\calU}{{\cal U}}
\newcommand{\calV}{{\cal V}}

\newcommand{\frakp}{{\mathfrak{p}}}
\newcommand{\frakq}{{\mathfrak{q}}}

\newcommand{\bfa}{\mathbf{a}}
\newcommand{\bfb}{\mathbf{b}}

\newcommand{\bfd}{\mathbf{d}}
\newcommand{\bfe}{\mathbf{e}}
\newcommand{\bff}{\mathbf{f}}
\newcommand{\bfg}{\mathbf{g}}

\newcommand{\bfp}{\mathbf{p}}
\newcommand{\bfq}{\mathbf{q}}
\newcommand{\bfr}{\mathbf{r}}
\newcommand{\bfs}{\mathbf{s}}

\newcommand{\bfu}{\mathbf{u}}
\newcommand{\bfv}{\mathbf{v}}
\newcommand{\bfw}{\mathbf{w}}
\newcommand{\bfx}{\mathbf{x}}
\newcommand{\bfy}{\mathbf{y}}
\newcommand{\bfz}{\mathbf{z}}

\newcommand{\bfzeta}{{\boldsymbol{\zeta}}}

\newcommand{\bftheta}{{\boldsymbol{\theta}}}

\newcommand{\bfpi}{{\boldsymbol{\pi}}}

\newcommand{\bfomega}{{\boldsymbol{\omega}}}

\newcommand{\bfell}{{\boldsymbol{\ell}}}

\newcommand{\bfB}{\mathbf{B}}

\newcommand{\bfE}{\mathbf{E}}

\newcommand{\bfG}{\mathbf{G}}

\newcommand{\bfI}{\mathbf{I}}
\newcommand{\bfJ}{\mathbf{J}}
\newcommand{\bfK}{\mathbf{K}}

\newcommand{\bfM}{\mathbf{M}}

\newcommand{\bfQ}{\mathbf{Q}}
\newcommand{\bfR}{\mathbf{R}}

\newcommand{\bbR}{\mathbb{R}}

%% file: cls/sym_zhl.tex
\newcommand*{\bfzero}{\mathbf{0}} 

\newcommand{\NEWTD}[1]{{\color{black}#1}}

\newcommand{\NEWZL}[1]{{\color{black}#1}}
\newcommand{\Revised}[1]{{\color{black}#1}}

\newcommand{\scaleMathLine}[2][1]{\resizebox{#1\linewidth}{!}{$\displaystyle{#2}$}}
\newcommand{\prl}[1]{\left(#1\right)}
\newcommand{\brl}[1]{\left[#1\right]}
\newcommand{\crl}[1]{\left\{#1\right\}}

\DeclareMathOperator{\tr}{tr}
\DeclareMathOperator*{\argmax}{argmax}   

\DeclareMathOperator*{\diag}{diag}

\DeclarePairedDelimiterX{\norm}[1]{\lVert}{\rVert}{#1}

\newcommand*{\intF}{\text{int}(\calF)}

\newcommand*{\lpg}{\bar{\bfg}}
\newcommand*{\LS}{\mathcal{LS}}



%% file: tex/introduction.tex
\section{Introduction}
Designing controllers that guarantee system stability and handle safety constraints is an important problem in safety-critical applications of automatic control systems, including autonomous transportation \cite{CBF_ames2014control, shalev2016safe}, robot locomotion \cite{CBF_ames2014rapidly}, and medical robotics \cite{yip2014model}. Safety depends on the system states, governed by the system dynamics, and the environment constraints. This leads to two requirements for designing provably safe controllers: an accurate model of the system dynamics and the satisfaction of safety constraints.

The first requirement has motivated data-driven dynamics learning approaches, utilizing machine learning techniques, such as Gaussian process (GP) regression \cite{deisenroth2015gp,kabzan2019learning,hewing2019cautious} and neural networks \cite{raissi2018multistep,chua2018deep}. For physical systems, recent works \cite{lutter2019deepunderactuated, zhong2020symplectic, duong21hamiltonian} design the model architecture to impose a Lagrangian or Hamiltonian formulation of the dynamics \cite{lurie2013analytical,HolmBook}, which a black-box model might struggle to infer. For Lagrangian dynamics, Lutter et al. \cite{lutter2019deepunderactuated} use neural networks to represent the mass and potential energy in 
the Euler-Lagrange equations of motion. Meanwhile, Zhong et al. \cite{zhong2020symplectic} use a differentiable neural ODE solver \cite{chen2018neural} to predict state trajectories of a Hamiltonian dynamics model, encoding Hamilton's equations of motion. A trajectory loss function is back-propagated through the ODE solver to update the Hamiltonian model parameters. Our prior work \cite{duong21hamiltonian} extends the neural ODE Hamiltonian formulation by imposing $SE(3)$ constraints to capture the kinematic evolution of rigid-body systems, such as ground or aerial robots. A Hamiltonian-based model architecture also allows the design of stable regulation or tracking controllers by energy shaping \cite{zhong2020symplectic, duong2021learning, duong21hamiltonian}. Interconnection and damping assignment passivity-based control (IDA-PBC) \cite{van2014port}, one of the main approaches for energy shaping, injects additional energy into the system via the control input to achieve a desired total energy, which is minimized at a desired regulation point. \NEWTD{Instead of learning robot dynamics in continuous time, Saemundsson et al. \cite{saemundsson2020variational} design a variational integrator network to learn discrete-time Lagrangian dynamics. Havens and Chowdhary \cite{havens2021forced} extend it by including control input in the model and use model predictive control for stabilization.}

The second requirement, ensuring satisfaction of safety constraints, has gained significant attention in planning and control. Model predictive control (MPC) methods \cite{borrelli_MPC_book,NMPC_book,MPC_Automatica_Camacho2006_nlin,MPC_mayne2000constrained} include safety constraints in an optimization problem, which is typically solved by discretizing time and linearizing the system dynamics.
Reachability-based techniques \cite{herbert2017fastrack, Funnel_lib, RTD_kousik2018bridging, HJ_time_varying} are directly applicable to nonlinear systems and offer strong safety guarantees but many require solving a Hamilton-Jacobi partial differential equation (PDE) or sum-of-squares optimization program. This is computationally challenging, especially for high-dimensional systems, and may require system decomposition techniques \cite{HJA_decomposition_chen2017}.
%
Control barrier functions (CBFs) \cite{CBF_ames2014control,CBF_ames2017TAC,CBF_ames2019ECC} offer an
elegant approach to encode safety constraints. For a control-affine nonlinear system, a CBF constraint is affine in the control input, allowing safe control synthesis with quadratic programming (QP) \cite{CBF_ames2019ECC}. However, \NEWZL{constructing} valid CBFs that guarantee the feasibility of the QP problem is challenging \cite{CBF_xu2018constrained, CBF_Dimitra_FxTS}. Given a stabilizing controller, reference governor techniques \cite{RG_bemporad1998reference,RG_kolmanovsky2014ACC_tutorial,RG_garone2016_ERG} use a virtual governor system to introduce safety constraints based on the Lyapunov function of the closed-loop system. Recent work \cite{RG_Omur_ICRA17, Li_SafeControl_ICRA20} achieves safe trajectory tracking in unknown environments but is limited to linear or feedback-linearizable systems.

Safe control synthesis with a learned model of the system dynamics needs to account for the model estimation error between the learned and the ground-truth dynamics \cite{brunke2021safe}. Model uncertainty may be viewed as a disturbance, applied to the learned system, and handled using robust or adaptive control techniques \cite{ioannou1996robust, gahlawat20al1adaptive, lee2013nonlinear, nageshrao2015port, ryalat2018robust, HJB_akametalu2018minimum}. Safety constraint satisfaction in the presence of model uncertainty can be achieved using robust MPC \cite{MPC_tube_gao2014tube, ostafew2016robust, hewing2019cautious, fan2020deep}, $\calL$-1 adaptive control \cite{gahlawat20al1adaptive}, or model reference adaptive control \cite{joshi2019deep} that tracks the trajectory of a reference model and compensates for model uncertainty. For example, Hewing et al. \cite{hewing2019cautious} propose an MPC technique that trains a Gaussian Process (GP) model of the system dynamics and use the GP uncertainty to introduce probabilistic safety constraints in the MPC optimization. 
Robust controllers for Hamiltonian systems may also be developed using the IDA-PBC approach \cite{ryalat2018robust, ryalat2015integral, PBC_acosta2014robust}. Most techniques have considered systems with states defined in Euclidean space and cannot handle manifold constraints, e.g., due to the orientation kinematics of a mobile robot.
For quadrotors, Lee et al. \cite{lee2013nonlinear} estimate disturbances from the tracking error and design a robust geometric controller but do not consider safety constraints.

In this paper, we consider both dynamics model learning and safe control synthesis for rigid-body systems, whose states include position, orientation, and generalized velocity. We assume that system has an unknown dynamics model but, as a physical system, it satisfies Hamilton's equations of motion over the $SE(3)$ manifold of positions and orientations. Given state-control trajectories, from past experiments or collected by a human operator, we seek to learn the system dynamics and design a tracking control law that handles safety constraints, e.g., obtained from distance measurements to obstacles in the environment. In our preliminary work \cite{refgovham_extended}, we learn a translation-equivariant Hamiltonian model of the system dynamics using a physics-guided neural ODE network \cite{duong21hamiltonian}. We use the Hamiltonian model to synthesize an energy-shaping geometric tracking controller. The total energy of the system serves as a Lyapunov function and enables us to enforce safety constraints during trajectory tracking using a reference governor to regulate the difference between the system energy and the distance to safety violation. However, our preliminary work \cite{refgovham_extended} uses the learned Hamiltonian model as the ground-truth dynamics and ignores the model estimation error in the control design. In this paper, we capture the estimation error as a bounded disturbance applied to the learned system and develop a robust safe tracking controller that takes the disturbance into account in the design of the reference governor. Our Hamiltonian dynamics learning and tracking control techniques are compared to a GP MPC technique \cite{hewing2019cautious} and are demonstrated in a 3D environment using a simulated hexarotor robot to achieve collision-free autonomous navigation.

In summary, the \textbf{contribution} of this work is a tracking control design for Hamiltonian systems with learned dynamics, which achieves \emph{robustness} to model estimation errors and \emph{safety} with respect to state constraints.

%% file: tex/problem.tex
\section{Problem Statement}
\label{sec:problem}

Consider a rigid body with position $\bfp \in \bbR^3$, orientation $\bfR\in SO(3)$, body-frame linear velocity $\bfv \in \bbR^3$, and body-frame angular velocity $\bfomega \in \bbR^3$. \NEWTD{Let $\frakq = [\bfp^\top\;\; \bfr_1^\top\;\; \bfr_2^\top\;\; \bfr_3^\top]^\top$ $\in SE(3)$ denote the body's generalized coordinates}, where $\bfr_1$, $\bfr_2$, $\bfr_3 \in \bbR^3$ are the rows of the rotation matrix $\bfR$. Let $\bfzeta = [\bfv^\top\;\;\bfomega^\top]^\top \in \bbR^6$ denote the body's generalized velocity. The generalized momentum $\frakp$ of the body is defined as:
\begin{equation}
\label{eq:momenta_Mtwist}
\frakp = \bfM(\frakq)\bfzeta \in \bbR^6,
\end{equation} 
where $\bfM(\frakq) \succ 0$ is the positive-definite generalized mass matrix. Let $\bfx = (\frakq, \frakp) \in T^*SE(3)$ denote the state of the rigid body system on the cotangent bundle $T^*SE(3)$ of the $SE(3)$ manifold. The Hamiltonian, $\mathcal{H}(\mathbf\frakq, \mathbf\frakp)$, captures the total energy of the system as the sum of the kinetic energy $\calT(\mathbf\frakq, \mathbf\frakp) =  \frac{1}{2}\mathbf\frakp^\top \bfM^{-1}(\frakq) \mathbf\frakp$ and the potential energy $\calU(\frakq)$:
\begin{equation}
\label{eq:hamiltonian_def}
\mathcal{H}(\mathbf\frakq, \mathbf\frakp) = \calT(\mathbf\frakq, \mathbf\frakp) +\calU(\mathbf\frakq). 
\end{equation}
%
The evolution of the state $\bfx$ is governed by Hamilton's equations of motion \cite{lee2017global}:
\begin{equation} \label{eq:Ham_sys_pf_nlin}
\begin{aligned}
\dot{\bfx} &= \bff(\bfx) + \bfG(\bfx) \bfu, \qquad\qquad \bfx(t_0) = \bfx_0,\\
           &= \begin{bmatrix}
\bf0 & \mathbf\frakq^{\times} \\
-\mathbf\frakq^{\times\top} & \mathbf\frakp^{\times} 
\end{bmatrix}
\begin{bmatrix}
\nabla_\frakq \mathcal{H}(\mathbf\frakq, \mathbf\frakp) \\
\nabla_\frakp \mathcal{H}(\mathbf\frakq, \mathbf\frakp) \end{bmatrix} + \begin{bmatrix} \bf0 \\ \bfB(\mathbf\frakq) \end{bmatrix} \bfu
\end{aligned}
\end{equation}
where $\bfu \in \bbR^6$ is the control input, \NEWTD{e.g. force and torque or motor speeds for a UAV system}, $\bfB(\mathbf\frakq)\in \bbR^{6\times 6}$ is an input gain matrix, and the operators $\mathbf\frakq^{\times}$, $\mathbf\frakp^{\times}$ are defined as:
\begin{equation*}
\mathbf\frakq^{\times} = \begin{bmatrix}
\bfR^\top\!\!\!\! & \bf0 & \bf0 & \bf0 \\
\bf0 & \hat{\bfr}_1^\top & \hat{\bfr}_2^\top & \hat{\bfr}_3^\top
\end{bmatrix}^\top\!\!\!\!, \quad \mathbf\frakp^{\times} = \begin{bmatrix} \mathbf\frakp_{\bfv}\\\mathbf\frakp_{\bfomega}\end{bmatrix}^{\times} \!\!\!\!= \begin{bmatrix}
\bf0 & \hat{\mathbf\frakp}_{\bfv}\\
\hat{\mathbf\frakp}_{\bfv} & \hat{\mathbf\frakp}_{\bfomega},
\end{bmatrix},
\end{equation*}
where the hat map $\hat{\bfw}$ for $\bfw \in \mathbb{R}^3$ is:
\begin{equation*}
\hat{\bfw} = \begin{bmatrix}
	0 & -w_3 & w_2 \\
	w_3 & 0 & -w_1 \\
	-w_2 & w_1 & 0 
\end{bmatrix}.
\end{equation*}
\NEWTD{The Hamiltonian dynamics model in \eqref{eq:Ham_sys_pf_nlin} can be extended to include energy dissipation in a port-Hamiltonian formulation \cite{van2014port} such as friction or drag forces \cite{zhong2020dissipative}. However, for clarity of the control design, we leave this for future work.}

We consider the case that the parameters of the Hamiltonian dynamics model in \eqref{eq:Ham_sys_pf_nlin}, including the mass $\bfM(\frakq)$, the potential energy $\calU(\frakq)$, and the input matrix $\bfB(\frakq)$, are unknown. Instead, we are given a trajectory dataset $\calD = \{t^{(i)}_{0:N}, \frakq^{(i)}_{0:N},\bfzeta^{(i)}_{0:N}, \bfu_{0:N-1}^{(i)}\}_{i = 1}^{D}$ consisting of $D$ sequences of generalized coordinates and velocities $(\frakq^{(i)}_{0:N},\bfzeta^{(i)}_{0:N})$ at times $t_0^{(i)} < t_1^{(i)} < \ldots < t_N^{(i)}$, collected by applying a constant control input $\bfu_n^{(i)}$ to the system with initial condition $(\frakq^{(i)}_{n},\bfzeta^{(i)}_{n})$ for $t \in [t_n,t_{n+1})$ and $n = 0,\ldots,N-1$. Our objective is to learn the system dynamics from the data set $\calD$ and design a control policy $\bfu = \bfpi(\bfx)$ such that the system follows a desired reference path without violating safety constraints. Let $\calO \subset \bbR^3$ and $\calF \coloneqq \bbR^{3} \setminus \calO$ denote the unsafe (obstacle) set and the safe (obstacle-free) set, respectively. Denote the interior of $\calF$ as $\intF$. We assume that $\calO$ is not known a priori but the distance $\bar{d}(\bfp, \calO)$ from the system's position $\bfp$ to $\calO$ can be sensed with a limited sensing range $d_{max} > 0$:
\begin{equation} \label{eq:dyO}
\bar{d}(\bfp , \calO) \coloneqq \min \crl{d(\bfp, \calO), d_{max}},
\end{equation}
where $d(\bfp, \calO) \coloneqq \inf_{\bfa \in \calO} \norm{\bfp - \bfa}$ denotes the Euclidean distance from $\bfp$ to the set $\calO$. The safe tracking control problem considered in this paper is summarized below.

\begin{problem}\label{prob:main_prob}
Let $\calD = \{t^{(i)}_{0:N}, \frakq^{(i)}_{0:N},\bfzeta^{(i)}_{0:N}, \bfu_{0:N-1}^{(i)}\}_{i = 1}^{D}$ be a training dataset of state-control trajectories obtained from a rigid-body system with unknown Hamiltonian dynamics in \eqref{eq:Ham_sys_pf_nlin}. Let $\bfr: \brl{0,1} \mapsto \text{Int}\prl{\calF}$ be a continuous function specifying a desired position reference path for the system. Assume that the reference path starts at the initial position at time $t_0$, i.e., $\bfr(0) = \bfp(t_0) \in \text{Int}\prl{\calF}$. Using local distance observations $\bar{d}(\bfp(t),\calO)$ of the unsafe set $\calO$, design a control policy $\bfpi: T^*SE(3) \mapsto \bbR^6$ so that the position $\bfp(t)$ of the closed-loop system with control law $\bfu = \bfpi(\bfx)$~converges asymptotically to $\bfr(1)$, while remaining safe, i.e., $\bfp(t) \in \calF, \forall t \geq t_0$. 
\end{problem}

%% file: tex/tech_dyn_learning.tex
\section{Learning $SE(3)$ Hamiltonian Dynamics from Data} 
\label{subsec:ham_dyn_learning}
In this section, \NEWTD{we design a dynamics model that can be learned from a previously collected trajectory dataset, e.g., obtained from manual control, and is sufficiently general to represent different mobile robots, such as cars and drones.} We describe how to learn Hamiltonian dynamics from the dataset $\calD = \{t^{(i)}_{0:N}, \frakq^{(i)}_{0:N},\bfzeta^{(i)}_{0:N}, \bfu_{0:N-1}^{(i)}\}_{i = 1}^{D}$, described in Sec. \ref{sec:problem}, using translation-equivariant Hamiltonian-based neural ODE networks \cite{duong21hamiltonian}. The mass $\bfM(\frakq)$, the potential energy $\calU(\frakq)$ and the input gain $\bfB(\frakq)$ are approximated by neural networks. We show that the model estimation errors caused by the trained neural networks can be considered as a disturbance applied on the learned system.

\subsection{Translation-equivariant $SE(3)$ Hamiltonian dynamics learning}
\label{subsec:ham_neural_ode}
Since the system dynamics does not change if we shift the position $\bfp$ to any points in the world frame, we offset the trajectories in the dataset $\calD$ so that they start from the position $\bf0$ and learn the system dynamics well around the origin. This is sufficient for stabilization task, e.g. using the controller design in Sec. \ref{sec:ham_control}, because driving the system from state $\bfx$ with position $\bfp$ to a desired state $\bfx^*$ with position $\bfp^*$ is the same as driving the system from the state $\bfx$ with position $\bf0$ to a desired state $\bfx^*$ with offset position $\bfp^* - \bfp$.

Since the momentum $\frakp$ is not directly available from the dataset $\calD$, we use the time derivative of the generalized velocity, derived from \eqref{eq:momenta_Mtwist}:
\begin{equation}
\label{eq:hamiltonian_zetadot}
\dot{\bfzeta} =  \prl{ \frac{d}{dt} \bfM^{-1}(\mathbf\frakq) }\mathbf\frakp + \bfM^{-1}(\mathbf\frakq)\dot{\mathbf\frakp}.
\end{equation}
Eq. \eqref{eq:Ham_sys_pf_nlin} and \eqref{eq:hamiltonian_zetadot} describe the Hamiltonian dynamics of the generalized coordinates and velocities with unknown inverse generalized mass matrix $\bfM^{-1}(\frakq)$, input matrix $\bfB(\frakq)$, and potential energy $\calU(\frakq)$, for which we aim to approximate by three neural networks $\bfM_\bftheta^{-1}(\frakq), \bfB_\bftheta(\frakq)$ and $\calU_\bftheta(\frakq)$, respectively, with parameters $\bftheta$.

To optimize for the parameters $\bftheta$, we use the Hamiltonian-based neural ODE framework that encodes the Hamiltonian dynamics \eqref{eq:Ham_sys_pf_nlin} and \eqref{eq:hamiltonian_zetadot} with $\bfM_\bftheta(\frakq), \bfB_\bftheta(\frakq)$ and $\calU_\bftheta(\frakq)$ in the network structure (Fig. \ref{fig:net_arch}). The forward pass rolls out the dynamics $\bar{\bff}_\bftheta$ described by \eqref{eq:Ham_sys_pf_nlin} and \eqref{eq:hamiltonian_zetadot} with the neural networks $\bfM_\bftheta(\frakq), \bfB_\bftheta(\frakq)$ and $\calU_\bftheta(\frakq)$ using a neural ODE solver (\cite{chen2018neural}) with initial state $({\frakq}^{(i)}_{n}, {\bfzeta}^{(i)}_{n})$. We obtain a predicted state $(\bar{\frakq}^{(i)}_{n+1}, \bar{\bfzeta}^{(i)}_{n+1})$ at times $t_{n+1}^{(i)}$ for each $n = 0, \ldots, N-1$ and $i = 1,\ldots, D$ as:
\begin{equation*}
    (\bar{\frakq}^{(i)}_{n+1}, \bar{\bfzeta}^{(i)}_{n+1}) = \text{ODESolver}\prl{({\frakq}^{(i)}_{n}, {\bfzeta}^{(i)}_{n}), \bar{\bff}, t_{n+1}^{(i)} - t_{n}^{(i)}; \bftheta}.
\end{equation*}
The loss function is defined as $\calL = \sum_{i = 1}^D \sum_{n = 1}^N c(\frakq^{(i)}_{n},\bfzeta^{(i)}_{n}, \bar{\frakq}^{(i)}_{n}, \bar{\bfzeta}^{(i)}_{n})$ where the distance metric $c$ is defined as the sum of position, orientation, and velocity errors on the tangent bundle $TSE(3)$: 
\begin{equation}
c \prl{\frakq,\bfzeta, \bar{\frakq}, \bar{\bfzeta} }= c_{\bfp}(\bfp,\bar{\bfp}) + c_{\bfR}(\bfR,\bar{\bfR}) + c_{\bfzeta}(\bfzeta,\bar{\bfzeta}),
\end{equation}
with the position error $c_{\bfp}(\bfp,\bar{\bfp}) = \| \bfp - \bar{\bfp}\|^2_2$, the velocity error $c_{\bfzeta}(\bfzeta,\bar{\bfzeta}) = \| \bfzeta - \bar{\bfzeta}\|^2_2$, and the rotation error $c_{\bfR}(\bfR,\bar{\bfR}) = \|\left(\log (\bar{\bfR} \bfR^\top)\right)^{\vee} \|_2^2$. The $\log$-map $\log (\cdot): SE(3) \mapsto \mathfrak{so}(3)$ returns a skew-symmetric matrix in $\mathfrak{so}(3)$ from a rotation matrix in $SE(3)$, and the $\vee$-map $(\cdot)^\vee : \mathfrak{so}(3) \mapsto \bbR^3$ is the inverse of the hat map $\hat{(\cdot)}$ in Sec. \ref{sec:problem}. 


The network parameters $\bftheta$ are optimized using gradient descent by back-propagating the gradient $\nabla_\bftheta\calL$ of the loss through the neural ODE solver efficiently using adjoint method \cite{chen2018neural}. Specifically, let $\bfa =\nabla_{\frakq, \bfzeta}\calL$ be the adjoint state and $\bfs = \prl{(\frakq, \bfzeta), \bfa, \nabla_\bftheta\calL}$ be the augmented state. The augmented state dynamics are \cite{chen2018neural}:
\begin{equation}
\dot{\bfs} = \bar{\bff}_{\bfs} = \prl{\bar{\bff}_{\bftheta}, -\bfa^\top\nabla_{\frakq, \bfzeta}\bar{\bff}_{\bftheta}, -\bfa^\top\nabla_\bftheta \bar{\bff}_{\bftheta}}.
\end{equation}
We obtain the gradient $\nabla_\bftheta\calL$ by a single call to a reverse-time ODE solver starting from $\bfs_{n+1} = \bfs(t_{n+1})$:
\begin{equation}
\bfs_0 = \left(\bar{\bfx}_0, \bfa_0, \nabla_\bftheta\calL\right) = \text{ODESolver}(\bfs_{n+1}, \bar{\bff}_\bfs, t_{n+1} - t_n),
\end{equation}
for $n = 0, \ldots, N-1$, and update the parameters $\bftheta$ using gradient descent. Please refer to \cite{chen2018neural} for more details.



\begin{figure}[t]
\centering
        \centering
        \includegraphics[width=\linewidth]{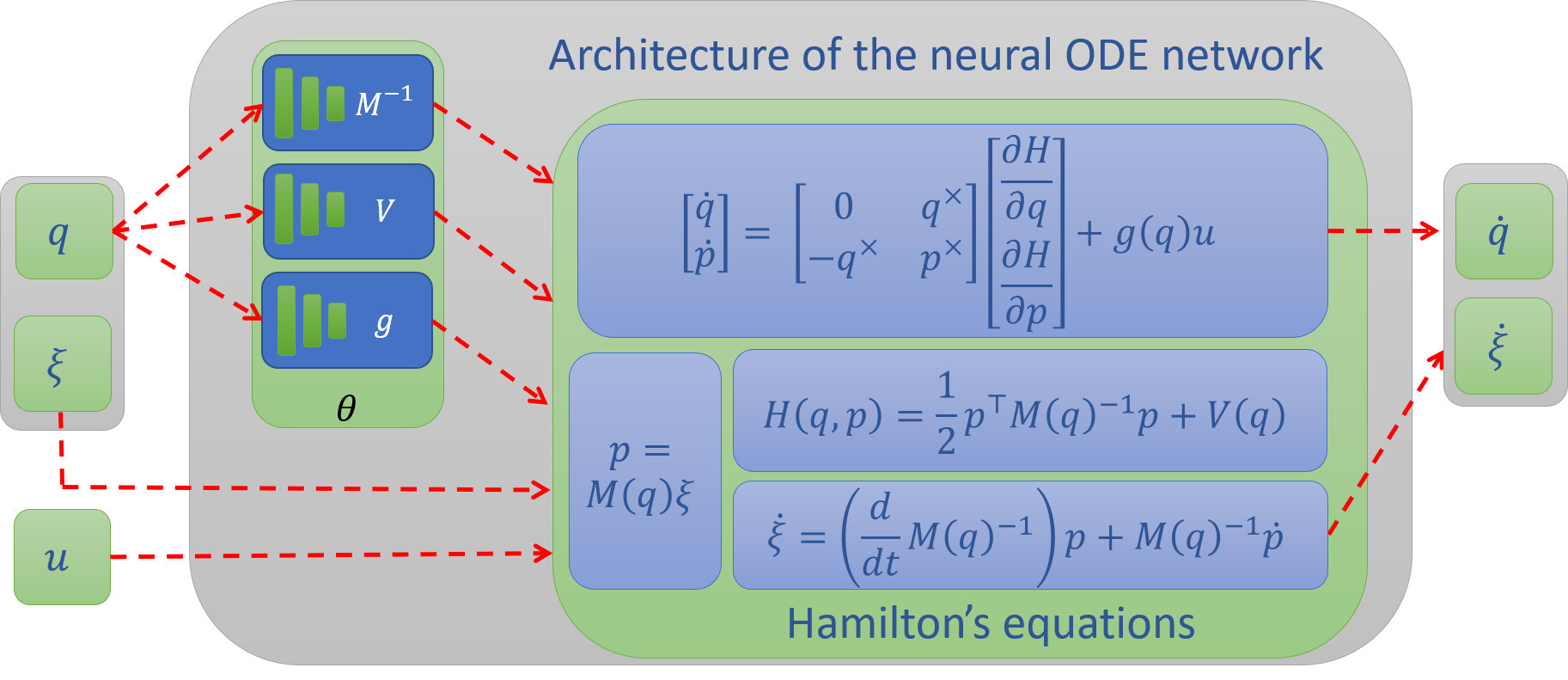}
\caption{Architecture of $SE(3)$ Hamiltonian neural ODE network.}
\label{fig:net_arch}
\end{figure}

\subsection{Model estimation error as a disturbance}
\label{sec:model_error}
Via the training process described in Sec. \ref{subsec:ham_neural_ode}, we approximate the ground truth mass $\tilde{\bfM}(\frakq)$, potential energy $\tilde{\calU}(\frakq)$ and input gain matrix $\tilde{\bfB}(\frakq)$ with the learned mass $\bfM_\bftheta(\frakq) = \tilde{\bfM}(\frakq) + \Delta\bfM_\bftheta(\frakq)$, potential energy $\calU(\frakq) = \tilde{\calU}(\frakq) + \Delta\calU_\bftheta(\frakq)$, and input gain $\bfB(\frakq) = \tilde{\bfB}(\frakq) + \Delta\bfB_\bftheta(\frakq)$ where $\Delta\bfM_\bftheta(\frakq), \Delta\calU_\bftheta(\frakq)$, and $\Delta\bfB_\bftheta(\frakq)$ are the estimation errors. We drop the subscript $\bftheta$ to simplify the notations. The generalized coordinates $\frakq$ and the ground-truth momentum $\tilde{\frakp} := \tilde{\bfM}(\frakq)\bfzeta$, satisfy the Hamiltonian dynamics \eqref{eq:Ham_sys_pf_nlin}:
\begin{equation}
    \label{eq:gt_ham_dyn}
    \begin{aligned}
        \dot{\frakq} &= \frakq^\times \nabla_{\tilde{\frakp}} \mathcal{\tilde{H}}(\mathbf\frakq, \tilde{\mathbf\frakp}) = \frakq^\times\bfzeta\\
        \dot{\tilde{\frakp}} &= -\frakq^{\times\top} \nabla_\frakq \mathcal{\tilde{H}}(\mathbf\frakq, \tilde{\mathbf\frakp}) + \tilde{\frakp}^\times\nabla_\frakp \mathcal{\tilde{H}}(\mathbf\frakq, \tilde{\mathbf\frakp}) + \tilde{\bfB}(\frakq)\bfu\\
        &= -\frakq^{\times\top} \nabla_\frakq \mathcal{\tilde{H}}(\mathbf\frakq, \tilde{\mathbf\frakp}) + \tilde{\frakp}^\times \bfzeta + \tilde{\bfB}(\frakq)\bfu, 
    \end{aligned}
\end{equation}
with the ground-truth Hamiltonian
\begin{equation}
    \tilde{\mathcal{H}}(\mathbf\frakq, \tilde{\mathbf\frakp}) =  \frac{1}{2}\tilde{\mathbf\frakp}^\top \tilde{\bfM}^{-1}(\mathbf\frakq) \tilde{\mathbf\frakp} + \tilde{\calU}(\mathbf\frakq) = \frac{1}{2}\bfzeta^\top \tilde{\bfM}(\mathbf\frakq)\bfzeta + \tilde{\calU}(\mathbf\frakq).
\end{equation}
Meanwhile, for the generalized coordinates $\frakq$ and the momentum $\frakp := \bfM(\frakq)\bfzeta$, the Hamiltonian dynamics is learned from data and of the form:
\begin{equation}
    \label{eq:learned_ham_dyn}
    \begin{aligned}
        \dot{\frakq} &= \frakq^\times \nabla_{{\frakp}} \mathcal{{H}}(\mathbf\frakq, {\mathbf\frakp}) = \frakq^\times\bfzeta\\
        \dot{{\frakp}} &= -\frakq^{\times\top} \nabla_\frakq \mathcal{{H}}(\mathbf\frakq, {\mathbf\frakp}) + {\frakp}^\times\nabla_\frakp \mathcal{{H}}(\mathbf\frakq, {\mathbf\frakp}) + {\bfB}(\frakq)\bfu\\
        &= -\frakq^{\times\top} \nabla_\frakq \mathcal{{H}}(\mathbf\frakq, {\mathbf\frakp}) + {\frakp}^\times \bfzeta + {\bfB}(\frakq)\bfu, 
    \end{aligned}
\end{equation}
with the learned Hamiltonian 
\begin{equation*}
    \mathcal{H}(\mathbf\frakq, \mathbf\frakp) = \frac{1}{2}\bfzeta^\top \bfM(\mathbf\frakq)\bfzeta + \calU(\mathbf\frakq) = \tilde{\calH}(\mathbf\frakq, \tilde{\mathbf\frakp}) + \Delta\calH(\mathbf\frakq, \mathbf\frakp),
\end{equation*}
and its estimation error $\Delta\calH(\mathbf\frakq, \mathbf\frakp) = \frac{1}{2}\bfzeta^\top \Delta\bfM(\mathbf\frakq)\bfzeta + \Delta\calU(\mathbf\frakq)$.

However, the learned dynamics \eqref{eq:learned_ham_dyn} is only an approximation of the actual dynamics for $(\frakq, \frakp)$. While the dynamics of $\frakq$ does not change, the actual dynamics of the learned momentum, $\frakp = \bfM(\frakq) \bfzeta =  \tilde{\frakp} + \Delta\frakp$, where $\Delta\frakp = \Delta\bfM(\frakq) \bfzeta$, is derived from \eqref{eq:gt_ham_dyn} as follows:
\begin{equation}
    \label{gt_learned_ham_dyn}
    \begin{aligned}
    \dot{\frakp} &= \dot{\tilde{\frakp}} + \dot{\Delta\frakp}  \\
    &= -\frakq^{\times\top} \nabla_\frakq \mathcal{H}(\mathbf\frakq, \mathbf\frakp) +  {\frakp}^\times \bfzeta + {\bfB}(\frakq)\bfu \\
    & \quad + \frakq^{\times\top} \nabla_\frakq \prl{\Delta\mathcal{H}(\mathbf\frakq, {\mathbf\frakp})} - \Delta{\frakp}^\times \bfzeta - \Delta{\bfB}(\frakq)\bfu + \dot{\Delta\frakp}. \\
    & = -\frakq^{\times\top}\nabla_\frakq \mathcal{H}(\mathbf\frakq, \mathbf\frakp) +  {\frakp}^\times \bfzeta + {\bfB}(\frakq)\bfu + \bfd_1,
    \end{aligned}
\end{equation}
where the force 
\begin{equation}
    \label{eq:model_error_disturbance}
    \bfd_1~:=~\frakq^{\times\top} \nabla_\frakq \prl{\Delta\mathcal{H}(\mathbf\frakq, {\mathbf\frakp})} - \Delta{\frakp}^\times \bfzeta - \Delta{\bfB}(\frakq)\bfu + \dot{\Delta\frakp}
\end{equation}

represents the effect of the model errors $\Delta\bfM(\frakq), \Delta\calU(\frakq)$, and $\Delta\bfB(\frakq)$ and is considered as a disturbance applied to the learned system \eqref{eq:learned_ham_dyn}. \Revised{To improve the error $\bfd_1$ with respect to the position $\bfp$, we enforce translation-equivariance in the neural ODE model, as described in Sec. III.A, and learn the model well around the origin. This allows us to offset any position $\bfp$ to the well-learned region around the origin. To reduce the model error with respect to orientation, we collect a training dataset that covers different regions of roll, pitch, and yaw angles, e.g. by manually driving a UAV to different desired positions and yaw angles. A promising approach to estimate the disturbance magnitude is to employ a Bayesian formulation of the neural ODE network used to learn the dynamics model. A Bayesian model will provide a posterior distribution, rather than point estimates, for the model parameters (i.e. $\bfM^{-1}(\frakq)$, $\bfB(\frakq)$, and $\calU(\frakq)$), whose variance can be used to obtain parameter error bounds and, in turn, a disturbance bound. Bayesian neural network models that can be used for dynamics learning include Bayesian neural ODE networks \cite{dandekar2020bayesian, xu22infinitely}, neural stochastic differential equation (SDE) networks \cite{liu2019neural}, or Gaussian-process ODEs \cite{heinonen18gpode}. This motivates analyzing the robustness of our control design in Sec.~\ref{sec:ham_control} to the disturbance $\bfd_1$ caused by the model errors.}

%% file: tex/tech_IDA_PBC_controller_v2.tex
\section{Stabilization of Hamiltonian Dynamics with Matched Disturbances}
\label{sec:ham_control}

As discussed in Sec.~\ref{subsec:ham_dyn_learning}.\ref{sec:model_error}, due to estimation errors in the dynamics learning process, the learned system model \Revised{satisfies} Hamilton's equations of motion in \eqref{eq:Ham_sys_pf_nlin} subject to a matched disturbance signal $\bfd_1: \bbR \mapsto \bbR^6$:
\begin{equation}
\label{eq:sys_pch}
\scaleMathLine[0.89]{\begin{bmatrix}
\dot{\frakq} \\
\dot{\frakp} \\
\end{bmatrix}
= \begin{bmatrix}
\bfzero & \frakq^{\times} \\
-\frakq^{\times\top}\!\! & \frakp^{\times} 
\end{bmatrix}
\begin{bmatrix}
\nabla_\frakq \mathcal{H}(\frakq, \frakp) \\
\nabla_\frakp \mathcal{H}(\frakq, \frakp) 
\end{bmatrix} 
+ \begin{bmatrix} \bfzero \\ \bfB(\frakq) \end{bmatrix}\bfu 
+ \begin{bmatrix} \bfzero \\ \bfd_1 \end{bmatrix}.}
\end{equation}
We consider a passivity-based stabilizing controller for \eqref{eq:sys_pch}, and analyze its robustness with respect to the disturbance signal $\bfd_1$ and its safety with respect to the obstacle set $\calO$.

\subsection{Passivity-based control}
\label{subsec:ham_control}

Consider a desired regulation point $\bfx^* = (\frakq^*,\frakp^*)$ for the system in \eqref{eq:sys_pch} with generalized coordinates $\frakq^* = (\bfp^*, \bfR^*)$ and momentum $\frakp^* = 0$. Since the Hamiltonian $\calH(\bfx)$ may not have a minimum at $\bfx^*$, the control signal $\bfu$ in \eqref{eq:sys_pch} should be designed to inject additional energy $\calH_a (\bfx, \bfx^*)$ into system and achieve a desired Hamiltonian $\calH_d (\bfx, \bfx^*) = \calH(\bfx) + \calH_a (\bfx, \bfx^*)$, which is minimized at $\bfx^*$. This is the approach followed by interconnection and damping assignment passivity-based control (IDA-PBC) \cite{ortega2002interconnection-IDAPBC}. Let $\bfx_e =(\frakq_e, \frakp_e)$ denote the error in generalized coordinates and momentum: 
\begin{equation}\label{eq:error_state}
\begin{aligned}
\bfR_e& = \bfR^{*\top} \bfR =  \begin{bmatrix} \bfr_{e1} & \bfr_{e2} & \bfr_{e3} \end{bmatrix}^\top &\bfp_e &= \bfp - \bfp^*\\
\frakq_e &= \begin{bmatrix} \bfp_e^\top & \bfr_{e1}^\top & \bfr_{e2}^\top & \bfr_{e3}^\top \end{bmatrix}^\top &\frakp_e& = \frakp-\frakp^*.
\end{aligned}
\end{equation}
A possible choice of $\calH_d (\bfx, \bfx^*)$, minimized at $\bfx = \bfx^*$, is:
\begin{align}
\calH_d (\bfx, \bfx^*) &=  \calT(\frakq_e,\frakp_e) + \calU_d(\frakq_e) \label{eq:desired_hamiltonian}\\
&=\frac{1}{2} \frakp_e^\top\bfM^{-1}(\frakq_e) \frakp_e
+ \frac{k_{\bfp}}{2} \norm{\bfp_e}^2
+ \frac{k_{\bfR} }{2} \tr(\bfI - \bfR_e),\notag
\end{align}
where $k_\bfp$ and $k_{\bfR}$ are positive scalars. 

\Revised{The IDA-PBC method \cite{duong21hamiltonian,ortega2002interconnection-IDAPBC} designs a controller $\bfu = \bfpi(\bfx,\bfx^*)$ such that the closed-loop dynamics of the system in \eqref{eq:sys_pch} are governed by the desired Hamiltonian in \eqref{eq:desired_hamiltonian} as:
\begin{equation} \label{eq:sys_pch_dsr}
\begin{bmatrix}
\dot{\mathbf\frakq}_e \\
\dot{\mathbf\frakp}_e \\
\end{bmatrix}
= \begin{bmatrix}
\bfzero & \bfJ(\bfx,\bfx^*) \\
-\bfJ(\bfx,\bfx^*)^\top & -\bfK_d
\end{bmatrix}
\begin{bmatrix}
\nabla_{\frakq_e} \calH_d(\bfx,\bfx^*) \\
\nabla_{\frakp_e} \calH_d(\bfx,\bfx^*)
\end{bmatrix} + \begin{bmatrix} \bfzero \\ \bfd \end{bmatrix},
\end{equation}
where the terms $\bfJ(\bfx,\bfx^*)$, $\bfK_d$, and $\bfd$ in the transformed dynamics depend on the control design. To obtain the controller, one uses the relationship between $\bfx$ and $\bfx_e$ in \eqref{eq:error_state} to equate the dynamics in \eqref{eq:sys_pch} and \eqref{eq:sys_pch_dsr}, leading to:
\begin{equation} \label{eq:controller_IDA_PBC}
\bfu = \bfpi(\bfx,\bfx^*) = \bfB^{\dagger}(\frakq)  \bfb(\bfx,\bfx^*),
\end{equation}
where $\bfB^{\dagger}(\frakq) = (\bfB^\top (\frakq) \bfB(\frakq))^{-1}\bfB^\top(\frakq)$ is the pseudo-inverse of the input gain $\bfB(\frakq)$ and:
\begin{multline} \label{eq:desired_control_IDA_PBC}
\bfb(\bfx,\bfx^*) = \Big( \frakq^{\times \top} \nabla_{\frakq} \calH(\bfx) - \frakp^\times \nabla_\frakp \calH(\bfx) \\
 -\bfJ(\bfx,\bfx^*)^\top \nabla_{\frakq_e} \calH_d(\bfx, \bfx^*) - \bfK_d \nabla_{\frakp_e} \calH_d(\bfx, \bfx^*) \Big)
\end{multline}
with $\bfJ(\bfx, \bfx^*) \coloneqq \begin{bmatrix}
\bfR^\top\!\!\!\! & \bf0 & \bf0 & \bf0 \\
\bf0 & \hat{\bfr}_{e1}^\top & \hat{\bfr}_{e2}^\top & \hat{\bfr}_{e3}^\top
\end{bmatrix}^\top.$
If the IDA-PBC matching equations~\cite{blankenstein2002matching},
\begin{equation} \label{eq:matching_condition}
\bfB^{\perp}(\frakq) \bfb(\bfx,\bfx^*) = 0,
\end{equation}
are satisfied, where $\bfB^{\perp}(\frakq)$ is a maximal-rank left annihilator of $\bfB(\frakq)$, i.e., $\bfB^{\perp}(\frakq) \bfB(\frakq) = \bf0$, then the controller in \eqref{eq:controller_IDA_PBC} achieves the desired closed-loop dynamics in \eqref{eq:sys_pch_dsr} with $\bfd = \bfd_1$, i.e., without introducing any extra disturbance.

If the matching equations \eqref{eq:matching_condition} cannot be satisfied globally, i.e., the IDA-PBC controller does not solve the system $\bfB(\frakq) \bfu = \bfb(\bfx, \bfx^*)$ exactly, then $\bfpi(\bfx^*, \bfx) = \bfB^\dagger(\frakq) \bfb(\bfx, \bfx^*)$ is a least-squares solution. In this case, the residual,
\begin{equation} \label{eq:disturbance_d_2}
\bfd_2 \coloneqq \prl{\bfB(\frakq) \bfB^\dagger(\frakq) - \bfI} \bfb(\bfx, \bfx^*),
\end{equation}
is introduced as an additional matched disturbance:
\begin{equation} \label{eq:total-disturbance}
\bfd = \bfd_1 + \bfd_2
\end{equation}
in the closed-loop dynamics in \eqref{eq:sys_pch_dsr}. Since the magnitude of $\bfd_2$ is proportional to that of $\bfb(\bfx,\bfx^*)$, it depends on the desired regulation point $\bfx^*$. An underactuated quadrotor system example is provided in Sec.~\ref{sec:evaluation}.\ref{subsec:quadrotor-sim}.

In general, the matching equations \eqref{eq:matching_condition} are nonlinear PDEs and can be solved explicitly only for certain cases \cite{blankenstein2002matching}. If $\bfB(\frakq)$ is invertible, i.e., the system in \eqref{eq:sys_pch} is fully-actuated, then the solution in \eqref{eq:controller_IDA_PBC} exists and is unique. For systems with underactuation degree $1$, the matching equations may be reduced to ODEs with closed-form solution \cite{gomez2001stabilization} or solved with certain desired kinetic energy \cite{acosta2005interconnection}. Yuksel et al.~\cite{yuksel2014reshaping} solve the matching equations specifically for stabilizing a quadrotor system, using Euler angles instead of a rotation matrix. A survey on this topic is available in~\cite{blankenstein2002matching}.
}



%% file: tex/tech_IDA_PBC_robustness_v2.tex
\subsection{Robustness analysis}
\label{sec:robustness-analysis}

\Revised{In this section, we analyze the stability and robustness with respect to the disturbance signal $\bfd$ in \eqref{eq:total-disturbance} of the IDA-PBC controller in \eqref{eq:controller_IDA_PBC}. Although the techniques we developed for dynamics learning in Sec.~\ref{subsec:ham_dyn_learning} and control synthesis in Sec.~\ref{sec:ham_control}.\ref{subsec:ham_control} did not make any assumptions about the Hamiltonian system in \eqref{eq:sys_pch}, our robustness and safety analysis that follows is developed under two assumptions.

\begin{assumption} \label{asp:bounded_matched_d}
The disturbance signal \eqref{eq:total-disturbance} is uniformly bounded, i.e., $\norm{\bfd} \leq \delta_{\bfd}$ for some $\delta_{\bfd} > 0$. 
\end{assumption}

\begin{assumption} \label{asp:constant_M}
The generalized mass matrix is constant, i.e., $\bfM(\frakq) \equiv \bfM$.
\end{assumption}

Without Assumption~\ref{asp:bounded_matched_d}, it is not possible to provide any performance guarantees for the control design because the disturbance $\bfd$ can have an arbitrary effect on the evolution of the closed-loop system dynamics. The disturbance magnitude bound $\delta_\bfd$ exists if we assume bounded estimation errors, bounded velocity and acceleration, bounded $\nabla_\frakq \prl{\Delta\mathcal{H}(\mathbf\frakq, {\mathbf\frakp})} $, and bounded control input $\bfu$ from the controller \eqref{eq:controller_IDA_PBC}. 

Our robustness analysis in Thm.~\ref{thm:robust_IDA_PBC} below constructs an ISS-Lyapunov function \cite{sontag2008input} to handle the disturbance $\bfd$. Assumption~\ref{asp:constant_M} simplifies the proof that we have a valid ISS-Lyapunov function. Extending the analysis to handle a state-dependent mass $\bfM(\frakq)$ is left for future work.

We simplify the error dynamics \eqref{eq:sys_pch_dsr} by noting that:
\begin{equation*}
\scaleMathLine{\bfe(\bfx,\bfx^*) \coloneqq \bfJ(\bfx,\bfx^*)^\top\nabla_{\frakq_e} \calH_d(\bfx,\bfx^*) = \begin{bmatrix} k_\bfp \bfR^\top\bfp_e \\ 
\frac{1}{2}k_{\bfR}\prl{\bfR_e - \bfR_e^\top}^{\vee}\end{bmatrix},}    
\end{equation*}
which leads to:
\begin{equation} \label{eq:error-dynamics}
\begin{aligned}
\dot{\frakq}_e &= \bfJ(\bfx,\bfx^*)\bfM^{-1} \frakp_e,\\
\dot{\frakp}_e &= -\bfe(\bfx,\bfx^*) - \bfK_d\bfM^{-1} \frakp_e + \bfd.
\end{aligned} 
\end{equation}
}

\begin{theorem} \label{thm:robust_IDA_PBC} 
Consider the Hamiltonian system in \eqref{eq:sys_pch} with desired regulation point $\bfx^* = (\frakq^*, 0)$ and control law specified in \eqref{eq:controller_IDA_PBC} with parameters $k_\bfp$, $k_\bfR$, $\bfK_\bfd$. Assume that the initial state $\bfx(t_0)$ lies in the domain $\calA = \crl{ \bfx \mid \tr(\bfI - \bfR^{*\top}\bfR) \leq \alpha < 4, \|\frakp\| \leq \beta}$ for some positive constants $\alpha$ and $\beta$. Then, the function:
\begin{equation} \label{eq:ISS_lyap}
\calV(\bfx, \bfx^*) =  \calH_d(\bfx, \bfx^*) + \rho \frac{d}{dt} \calU_d(\frakq_e)
\end{equation}
is an ISS-Lyapunov function \cite{sontag2008input} with respect to $\bfd$ in \eqref{eq:total-disturbance} and satisfies:
\begin{equation} \label{eq:V_bounds} 
\begin{aligned} 
k_1 \norm{\bfz}^2 \leq \calV(\bfx, \bfx^*)  &\leq k_2 \norm{\bfz}^2, \\ 
\dot{\calV}(\bfx, \bfx^*) &\leq -k_3 \norm{\bfz}^2 + k_\gamma \delta_{\bfd}^2,
\end{aligned}
\end{equation}
where $\bfz \coloneqq [\|\bfe(\bfx,\bfx^*)\|\; \|\frakp_e\|]^\top \in \bbR^2$, $k_\gamma =\frac{1}{2 \lambda_{\min}(\bfK_\bfd)} +\frac{\rho \lambda_2^2}{2 \lambda_1}$, $\lambda_1 \coloneqq \NEWZL{\lambda_{\min}(\bfM^{-1})}$, $\lambda_2 \coloneqq \NEWZL{\lambda_{\max}(\bfM^{-1})}$, $k_1 = \frac{1}{2} \lambda_{\min}(\bfQ_1)$, $k_2 = \frac{1}{2} \lambda_{\max}(\bfQ_2)$, $k_3 = \frac{1}{2}\lambda_{\min}(\bfQ_3)$, and the associated matrices $\bfQ_1$, $\bfQ_2$, $\bfQ_3$ are defined as:
%
\begin{equation} \label{eq:Q-matrices}
\begin{aligned}
\bfQ_1 &= \begin{bmatrix} \min\crl{k_{\bfp}^{-1}, k_{\bfR}^{-1}} & - \rho \lambda_2\\
- \rho \lambda_2 & \lambda_1 \end{bmatrix},\\
\bfQ_2 &= \begin{bmatrix} \max\crl{k_{\bfp}^{-1}, \frac{4 k_{\bfR}^{-1}}{4 - \alpha}} & \rho \lambda_2\\
\rho \lambda_2 & \lambda_2 \end{bmatrix}, 
\;\bfQ_3 = \begin{bmatrix} q_1 & q_2 \\ q_2 & q_3\end{bmatrix},
\end{aligned}
\end{equation}
where the elements of $\bfQ_3$ are:
\begin{equation} \label{eq:Q3-elements}
\begin{aligned}
q_1 &= \rho \lambda_1,\\
q_2 &= - \rho \brl{ \lambda_{\max}(\bfM^{-1}\bfK_{\bfd}\bfM^{-1}) + \beta \lambda_2^2},\\ 
q_3 &= \lambda_{\min}(\bfK_\bfd)\lambda_1^2 - 2\rho \lambda_2^2 \max\crl{k_{\bfp},k_{\bfR}}.
\end{aligned}
\end{equation}
%
Denote the sub-level set of $\calV(\bfx, \bfx^*)$ with respect to positive scalar $c$ as: $\calS_{c} \coloneqq \crl{\bfx \mid \calV(\bfx, \bfx^*) \leq c}$.
%
%
Given constants $c_1$, $c_2$ defined as:
\begin{equation}
c_1 \coloneqq \frac{k_2 k_\gamma}{k_3} \delta_{\bfd}^2,\; c_2 \coloneqq k_1 \min \crl{k_\bfR^2 \alpha(4- \alpha)/4, \beta^2},
\end{equation}
$\calS_{c_2} \subseteq \calA $ is an estimate of the region of attraction of the control law in \eqref{eq:controller_IDA_PBC}. Any state $\bfx$ starting within $\calS_{c_2}$~will~converge exponentially to $\calS_{c_1}$ and remain within it. The position error trajectory $\bfp_e(t)$ is uniformly ultimately bounded as:
\begin{equation}\label{eq:uub}
\lim_{t \rightarrow \infty} \norm{\bfp_e(t)}^2 \leq \frac{c_1}{k_1 k_\bfp^2} = \frac{k_2k_\gamma}{k_1k_3 k_{\bfp}^2} \delta_\bfd^2.
\end{equation}
To ensure that $c_1 < c_2$, the disturbance bound $\delta_\bfd$ should satisfy $\delta_\bfd < \sqrt{\frac{c_2k_3}{k_2 k_\gamma}}$.
\end{theorem}

\begin{proof}
See Appendix~\ref{app:robust_IDA_PBC_proof}.
\end{proof}

The estimates of the region of attraction and the uniform ultimate bound on the position error \Revised{provided by Thm.~\ref{thm:robust_IDA_PBC} for the IDA-PBC controller} are conservative because our analysis considers the mass and inertia jointly as a generalized mass $\bfM$ and does not differentiate the force and torque disturbances. Besides considering separate disturbance bounds for the force and torque inputs, less conservative bounds can be achieved by introducing disturbance compensation \cite{lee2013nonlinear}.

%% file: tex/tech_safe_stab.tex
\subsection{Safety analysis} \label{sec:safety_analysis}

Sec.~\ref{sec:ham_control}.\ref{sec:robustness-analysis} analyzed the stability and robustness properties of the IDA-PBC controller for a given regulation point $\bfx^*$. Next, we use the Lyapunov function $\calV(\bfx,\bfx^*)$ in \eqref{eq:ISS_lyap} to derive conditions under which the trajectory of the closed-loop system remains outside the unsafe set $\calO$. We introduce a barrier function, which takes the region of attraction $\calS_{c_2}$ of the controller and the invariant set $\calS_{c_1}$ associated with the ultimate bound in Thm.~\ref{thm:robust_IDA_PBC} as well as the distance $\bar{d}(\bfp^*, \calO)$ to $\calO$ into account to quantify the margin to safety violation:
 \begin{multline} \label{eq:def_deltaE}
 \Delta E(\bfx, \bfx^*) \coloneqq \min \crl{c_2, k_1 k_\bfp^2 \bar{d}{\,}^2 \prl{\bfp^*, \calO}}  - \calV(\bfx, \bfx^*) \\
 + \max \crl{c_1 - \calV(\bfx, \bfx^*), 0},
 \end{multline}
where $k_1$, $k_\bfp$, $c_1$, $c_2$ are the constants specified in Thm.~\ref{thm:robust_IDA_PBC}. If, for a given regulation point $\bfx^*$, the safety margin $\Delta E(\bfx, \bfx^*)$ is positive initially, then any trajectory of the closed-loop system remains safe as it converges to the invariant set $\calS_{c_1}$.

\begin{proposition} \label{prop:safe_stab}
Consider the system in \eqref{eq:sys_pch} with regulation point $\bfx^* = (\frakq^*, 0)$ and control law in \eqref{eq:controller_IDA_PBC}. Suppose that the desired position $\bfp^*$ has sufficient clearance from the unsafe set $\calO$ and the disturbance $\bfd$ is bounded as follows:
\begin{equation}\label{eq:sufficient-clearance}
\bar{d}{\,}^2(\bfp^*, \calO) \geq \frac{k_2k_\gamma}{k_1k_3 k_{\bfp}^2} \delta_\bfd^2, \qquad \|\bfd\|^2 \leq \delta_{\bfd}^2 < \frac{c_2k_3}{k_2k_\gamma}.
\end{equation}
If the initial state $\bfx(t_0) = \bfx_0$ satisfies:
\begin{equation} \label{eq:safe_stab_cond}
\Delta E(\bfx_0, \bfx^*) \geq 0, 
\end{equation}
then the position error trajectory is uniformly ultimately bounded as in \eqref{eq:uub} and the system remains safe, i.e., $d(\bfp(t), \calO) \geq 0$ for all $t \geq t_0$.
\end{proposition}

\begin{proof} %
By the definition in \eqref{eq:def_deltaE}, $\Delta E(\bfx, \bfx^*) \geq 0$ implies that the Lyapunov function $\calV(\bfx, \bfx^*)$ satisfies one of three cases:
\begin{enumerate}
    \item $c_1  < 	\calV$,  $\calV \leq \min \crl{c_2, k_1 k_\bfp^2 \bar{d}{\,}^2 \prl{\bfp^*, \calO}}$,
    \item $c_1  \geq \calV$, $\calV \leq \min \crl{c_2, k_1 k_\bfp^2 \bar{d}{\,}^2 \prl{\bfp^*, \calO}}$,
    \item $c_1  \geq \calV$, $\calV > \min \crl{c_2, k_1 k_\bfp^2 \bar{d}{\,}^2 \prl{\bfp^*, \calO}}$.
\end{enumerate}

Case 3) can never happen because \eqref{eq:sufficient-clearance} implies that $c_1 \leq k_1 k_\bfp^2 \bar{d}{\,}^2 \prl{\bfp^*, \calO}$ and $c_1 < c_2$.

For Case 1), when $c_1 < \calV \leq c_2$, we know from Thm.~\ref{thm:robust_IDA_PBC} that $\dot{\calV} < 0$ and every trajectory starting in $\calS_{c_2}$ converges exponentially to $\calS_{c_1}$. In this case, from \eqref{eq:V_bounds}:
\begin{equation}
\begin{aligned}
k_1 k_\bfp^2 \bar{d}{\,}^2 \prl{\bfp^*, \calO} &\geq \calV(\bfx(t_0),\bfx^*) > \calV(\bfx(t),\bfx^*) \\
&\geq  k_1 \norm{\bfz(t)}^2 \geq k_1 k_\bfp^2 \norm{\bfp(t) - \bfp^*}^2.
\end{aligned}
\end{equation}
Therefore, $\norm{\bfp(t) - \bfp^*}^2 \leq \bar{d}{\,}^2 \prl{\bfp^*, \calO} \leq d^2(\bfp^*, \calO)$ and $d(\bfp(t), \calO) \geq 0$ for all $t \in \brl{t_0, t_1}$, where $t_1$ is the time when the trajectory enters $\calS_{c_1}$, corresponding to Case 2) above.

For Case 2), we have $\calV(\bfx,\bfx^*) \leq c_1$ since \eqref{eq:sufficient-clearance} implies that $c_1 < c_2$. From Thm.~\ref{thm:robust_IDA_PBC}, $\calS_{c_1}$ is forward invariant and:
\begin{equation}
    \norm{\bfp(t) - \bfp^*}^2 \leq \frac{\calV(\bfx(t),\bfx^*)}{k_1k_{\bfp}^2} \leq \frac{c_1}{k_1k_{\bfp}^2} = \frac{k_2k_\gamma}{k_1k_3 k_{\bfp}^2}\delta_\bfd^2.
\end{equation}
Hence, \eqref{eq:sufficient-clearance} implies that $d(\bfp(t), \calO) \geq 0$.
\end{proof}

%% file: tex/robust_ref_gvn.tex
\section{Safe and Stable Tracking using a Reference Governor}
\label{sec:tracking}

\begin{figure*}[t]
\centering
\begin{subfigure}
    \centering
    \includegraphics[height=35mm]{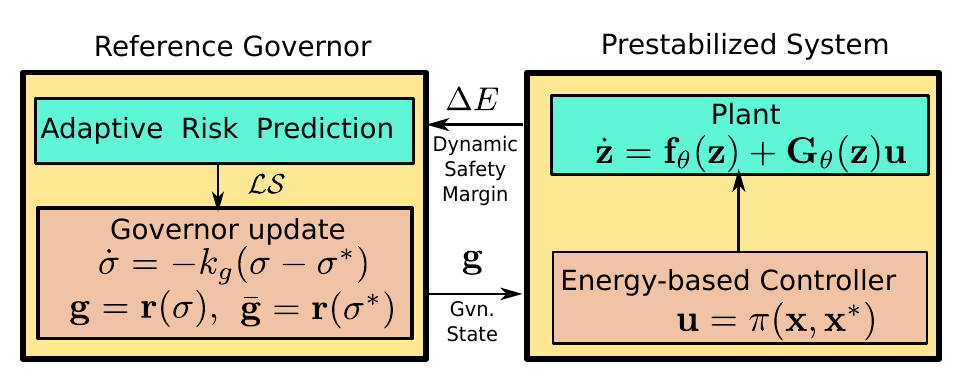}
\end{subfigure}%
\qquad \qquad
\begin{subfigure}
    \centering
    \includegraphics[height=35mm]{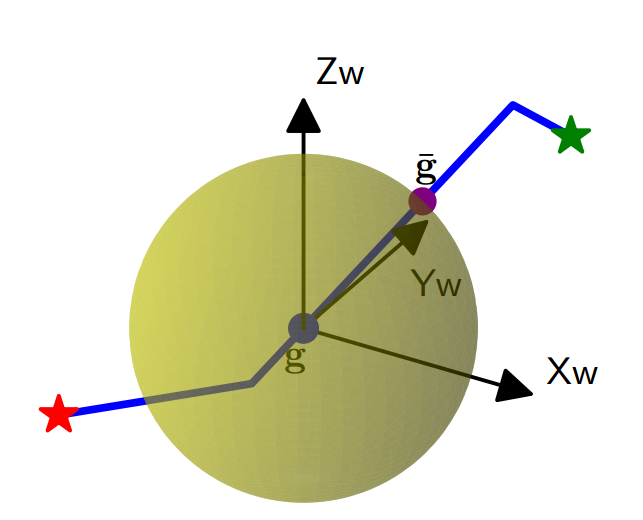}
\end{subfigure}%
\caption{Structure of the reference-governor tracking controller (left). A governor with state $\bfg$ adaptively tracks desired path $\bfr$ and generates a reference point $\bfx^* = \bfell(\bfg)$ for the closed-loop Hamiltonian system (right). A local projected goal $\lpg$ (purple dot) is generated as the farthest intersection between the local safe set $\LS(\bfx, \bfg)$ (yellow sphere) and the path $\bfr$ (blue curve) to guide the governor motion. }

\label{fig:structure}
\end{figure*}

In this section, we develop a safe tracking controller by introducing a reference governor \cite{RG_bemporad1998reference} to guide the reference point $\bfx^*$ for the stabilizing control law $\bfpi(\bfx,\bfx^*)$ in \eqref{eq:controller_IDA_PBC} along the desired reference path $\bfr$ introduced in Problem~\ref{prob:main_prob}. 

A reference governor is a virtual dynamical system whose state $\bfg(t)$ moves along $\bfr(\sigma)$ for $\sigma \in [0,1]$. In this paper, the governor state $\bfg(t) \in \mathbb{R}^{3}$ specifies a desired position $\bfp^*(t)$ for the Hamiltonian system. We introduce a lifting function $\bfx^*(t) = \bfell(\bfg(t))$ to provide a desired orientation $\bfR^*(t)$ and specify a reference state $\bfx^*(t)$ for the Hamiltonian system.

Given $\bfx^*(t)$, we compute the safety margin $\Delta E(\bfx(t), \bfx^*(t))$ in \eqref{eq:def_deltaE} and use the leeway amount by which the margin exceeds $0$ to move the governor state $\bfg(t)$ along $\bfr(\sigma)$. Intuitively, the reference point $\bfx^*(t) = \bfell(\bfg(t))$ speeds up when $\Delta E(\bfx(t), \bfx^*(t))$ increases, e.g., the distance to obstacles increases or the system energy function decreases, and vice versa. 


Given a point $\bfg = \bfr(\sigma)$ on the reference path for some $\sigma \in [0,1]$, we generate a reference state $\bfx^* = (\frakq^*, \frakp^*)$ where $\frakq^* = (\bfp^*, \bfR^*) = (\bfg, \bfI)$ and $\frakp^* = 0$. The governor state $\bfg$ represents the desired position $\bfp^*$ on the path. For simplicity, we set the desired rotation matrix $\bfR^* = \bfI$. If, in addition to $\bfr$, a desired yaw angle reference is provided, one can generate $\bfR^*$ using the method described in \cite{lee2010geometric} to achieve better orientation tracking. We define a lifting function $\bfell: \bbR^3 \mapsto T^*SE(3)$ that generates a reference state $\bfx^* = \bfell(\bfg)$ from the governor state $\bfg$:
\begin{equation} \label{eq:lift_funcs}
\bfell(\bfg) \coloneqq \begin{bmatrix}
\bfg^\top & \bfe_1^{\top} & \bfe_2^{\top} &\bfe_3^{\top} & \bf0^\top & \bf0^\top
\end{bmatrix}^\top,
\end{equation}
where $\bfe_1$, $\bfe_2$, $\bfe_3$ are the rows of the identity matrix. Given the reference state $\bfx^* = \bfell(\bfg)$, we compute the safety margin $\Delta E(\bfx, \bfx^*)$ in \eqref{eq:def_deltaE} and describe how to update the governor state to ensure that safety is preserved.

We update the governor state $\bfg(t) = \bfr(\sigma(t))$ along the path by regulating the parameter $\sigma$:
\begin{equation} \label{eq:gov_ctrl} 
\bfg(t) = \bfr(\sigma(t)), \quad \dot{\sigma}(t) = - k_g (\sigma(t) - \sigma^*(t)), 
\end{equation}
where $k_g > 0$ is a control gain and $\sigma^*(t) \in \brl{0, 1}$ is a desired time-varying parameter, which we construct using the safety margin $\Delta E(\bfx, \bfx^*)$. We require $\sigma^*(t)$ to satisfy two conditions: 1) always stay ahead of the current $\sigma(t)$: $\sigma^*(t) \geq \sigma(t)$, $\forall t \geq t_0$, and 2) have distance $\norm{\sigma^*(t) - \sigma(t)}$ proportional to $\Delta E(\bfx(t), \bfx^*(t))$. The first condition guarantees that the governor state $\bfg(t)$ moves forward along the path towards the goal $\bfr(1)$. The second condition allows the safety margin $\Delta E$ to adaptively regulate the governor state $\bfg(t)$ in order to ensure safety for the Hamiltonian system. To construct the desired path parameter $\sigma^*$, we define a \emph{local safe zone} as a ball around the governor state $\bfg$ with radius $\Delta E(\bfx, \bfx^*)$ based on the state $\bfx$ and the reference state $\bfx^* = \bfell(\bfg)$.

\begin{definition}\label{def:LS}
	A \emph{local safe zone} is a subset of $\bbR^3$ that depends on the system state $\bfx$ and the governor state $\bfg$:
	\begin{equation}
	\label{eq:LS}
	\LS(\bfx, \bfg) \!\coloneqq\! \crl{ \bfq \in \bbR^3 \!\mid\! \norm{\bfq- \bfg}^2 \leq \Delta E\!\prl{\bfx, \bfell(\bfg)}}\!,
	\end{equation}
	where $\bfell$ is the lifting function in \eqref{eq:lift_funcs} and $\Delta E$ is the safety margin in \eqref{eq:def_deltaE}.
\end{definition}

We determine $\sigma^*$ as the farthest intersection between the local safe zone $\LS(\bfx, \bfg)$ and the path $\bfr$ by solving the scalar optimization problem in Def. \ref{def:lpg}.

\begin{definition} \label{def:lpg}
	A \emph{local projected goal} for system-governor state $(\bfx,\bfg)$ is a point $\lpg \in \LS(\bfx, \bfg)$ that is farthest along the path $\bfr$:
	\begin{equation} \label{eq:lpg}
	\lpg = \bfr(\sigma^*),  \;\; \sigma^* = \argmax_{\sigma \in [0,1]} \crl{ \sigma \mid  \bfr(\sigma) \in \LS(\bfx, \bfg)}.
	\end{equation}	
\end{definition}

The construction of the local projected goal $\lpg$ is shown in Fig.~\ref{fig:structure} (right), showing a reference path $\bfr$, the local safe zone $\LS(\bfx, \bfg)$ and the local projected goal $\lpg$. This constructing of $\sigma^*$ and $\lpg$ completes the governor update law \eqref{eq:gov_ctrl}. 

Our safe tracking controller consists of the reference governor system in \eqref{eq:gov_ctrl}, adaptively updating the reference point $\bfx^* = \bfell(\bfg)$ via the lifting function in \eqref{eq:lift_funcs}, and the passivity-based controller in \eqref{eq:controller_IDA_PBC} that drives the Hamiltonian system towards $\bfx^*$. The stability, safety, and robustness of the proposed tracking controller are analyzed in Thm.~\ref{thm:main_result}.

\begin{theorem}\label{thm:main_result}
    Suppose that the desired path $\bfr(\sigma)$ has sufficient clearance from the unsafe set $\calO$ and the disturbance $\bfd$ is bounded as:
	\begin{equation*}
	\min_{\sigma \in \brl{0,1}} \bar{d}{\,}^2(\bfr(\sigma), \calO) \geq \frac{k_2k_\gamma}{k_1k_3 k_{\bfp}^2} \delta_\bfd^2, \quad \|\bfd\|^2 \leq \delta_{\bfd}^2 < \frac{c_2k_3}{k_2 k_\gamma}.
	\end{equation*}
	Consider the Hamiltonian system in \eqref{eq:sys_pch}, the governor system in \eqref{eq:gov_ctrl} with $\sigma^*$ constructed in Def.~\ref{def:lpg} and the control law $\bfu = \bfpi(\bfx, \bfell(\bfg))$ in \eqref{eq:controller_IDA_PBC}. Suppose that the initial state $(\bfx_0, \bfg_0)$ satisfies:
	\begin{equation} \label{eq:thm_ic}
	\Delta E \prl{\bfx_0, \bfell(\bfg_0)} > 0,\quad  \bfg_0 = \bfr(0) = \bfp(t_0),
	\end{equation}
	where $\Delta E(\bfx, \bfx^*)$ is the safety margin in \eqref{eq:def_deltaE}. The position $\bfp(t)$ of \eqref{eq:sys_pch} converges to a ball of radius $\sqrt{\frac{k_2k_\gamma}{k_1k_3 k_{\bfp}^2}} \delta_\bfd$ around $\bfr(1)$ and remains safe, i.e. $\bfp(t) \in \calF$, for all $t \geq t_0$
 
\end{theorem}

\begin{proof}
	To simplify notation, let $\Delta E(t) = \Delta E \prl{\bfx(t), \bfell(\bfg(t))}$. Initially, $\bfg_0 = \bfp(t_0) = \bfr(0) \in \LS(\bfx_0, \bfg_0)$ and $\Delta E(t_0) > 0$. The local projected goal $\lpg$ and the associated $\sigma^*$ are well defined in Def. \ref{def:lpg}. By the governor update law \eqref{eq:gov_ctrl}, the path parameter $\sigma$ increases, the governor state $\bfg(\sigma)$ moves along $\bfr$ towards the goal $\bfr(1)$. The desired state $\bfx^* = \bfell(\bfg)$ is updated via the lifting function \eqref{eq:lift_funcs}. As $\bfg$ tracks $\lpg$ on the path $\bfr$ via the path parameter update in \eqref{eq:gov_ctrl}, the system state $\bfx$ tracks $\bfx^* = \bfell(\bfg)$.
	During this process, the safety margin $\Delta E(t)$ fluctuates and regulates the rate of change of $\sigma$. 

    Since $\sigma^*(t)$ is bounded in \eqref{eq:lpg}, $\sigma(t)$ is updated continuously \cite{filippov1988differential} in \eqref{eq:gov_ctrl}, leading to a continuous governor state $\bfg(t)$. By construction, the lifting function $\bfell(\bfg)$ is continuous in $\bfg$. Therefore, the reference point $\bfx^*(t)=\bfell(\bfg(t))$ is continuous in time, leading to a continuous Lyapunov function $\calV(\bfx, \bfx^*)$ and a continuous safety margin $\Delta E(t)$. As a result, the safety margin $\Delta E(t)$ cannot become negative without crossing $0$ from above at some time $T_0$. As $\Delta E(t) \downarrow 0$, the local safe zone shrinks to a point, i.e., $\LS(\bfx, \bfg) \downarrow \crl{\bfg}$. This immediately stops the the governor because $\lpg = \bfg(T_0) = \bfr(\sigma(T_0))$ and $\dot{\sigma}(T_0) = 0$. 

	As a result, Proposition~\ref{prop:safe_stab} states that $\bfx(t)$ stays within the invariant set $\calS_{c_2} \prl{\bfx^*(T_0)}$ for $t \geq T_0$ and converges to $\calS_{c_1} \prl{\bfx^*(T_0)}$ without leaving $\calF$. Eq. \eqref{eq:def_deltaE} shows that $\Delta E(t) = 0$ implies $c_1 \leq \calV(t) \leq c_2$. By Thm.~\ref{thm:robust_IDA_PBC}, as $\bfx(t)$ approaches $\bfx^*(T_0)$, we have $\dot{\calV}(T_0) < 0$, i.e., the Lyapunov function $\calV$ is decreasing. There exists $h > 0$ such that $\Delta E(T_0 + h)$ becomes strictly positive. Hence, the governor is able to move again towards a new $\lpg$ generated by the positive $\Delta E(T_0 + h)$. This process continues until the governor state $\bfg(t)$ converges to $\bfr(1)$, the closed-loop system converges to the region  $\calS_{c_1} \prl{\bfell(\bfr(1))}$ and the position $\bfp(t)$ satisfies the uniform ultimate bound in \eqref{eq:uub} around $\bfr(1)$.
\end{proof}

\NEWTD{Note that while our control design does not account for state estimation errors, e.g. from an odometry algorithm with a sensor setup (e.g. stereo camera, LiDAR, or visual-inertial), we can conservatively handle the errors by reducing the obstacle distance $\bar{d}$ in the safety margin specification in \eqref{eq:sufficient-clearance}.}

%% file: tex/extension_Rn_underactuated.tex
\section{Application to Hamiltonian Dynamics in $\bbR^n$}
\label{sec:extension_Rn_underactuated}
\Revised{
In this section, we show that our control design can be easily modified and applied to a Hamiltonian system with configuration $\frakq$ in $\bbR^n$ and dynamics:
\begin{equation} \label{eq:sys_pch_Rn}
\begin{bmatrix}
\dot{\frakq} \\
\dot{\frakp} \\
\end{bmatrix} = \begin{bmatrix}
\bfzero & \bfI_n \\
-\bfI_n & \bfzero
\end{bmatrix} 
\begin{bmatrix}
\nabla_{\frakq} \calH(\frakq,\frakp) \\
\nabla_{\frakp} \calH(\frakq,\frakp)
\end{bmatrix}
+ 
\begin{bmatrix}
\bfzero \\
\bfB(\frakq)
\end{bmatrix} \bfu  + 
\begin{bmatrix}
\bfzero \\
\bfd_1
\end{bmatrix}
\end{equation}
where the Hamiltonian $\calH(\frakq, \frakp)$ is defined as:
\begin{equation}
\label{eq:hamiltonian_Rn}
\mathcal{H} (\frakq, \frakp) =  \frac{1}{2} \frakp^\top\bfM^{-1}(\frakq) \frakp + \calU(\frakq).
\end{equation}

Given a desired regulation point $\bfx^* = (\frakq^*,\frakp^*)$ with momentum $\frakp^* = 0$, define the error state $\bfx_e = (\frakq_e,\frakp_e)$ as:
\begin{equation}\label{eq:error_state_Rn}
\frakq_e = \frakq - \frakq^*, \qquad \frakp_e = \frakp-\frakp^*.
\end{equation}
A desired Hamiltonian, minimized at $\bfx = \bfx^*$, is:
\begin{equation}
\label{eq:desired_hamiltonian_Rn}
\mathcal{H}_d (\bfx, \bfx^*) =  \frac{1}{2} \frakp_e^\top\bfM^{-1}(\frakq_e) \frakp_e + \frac{k_\bfp}{2} \norm{\frakq_e}^2.
\end{equation}
The IDA-PBC controller:
\begin{equation}
    \label{eq:controller_IDA_PBC_Rn}
    \bfu = \bfpi(\bfx, \bfx^*) = \bfB^{\dagger}(\frakq) \bfb(\bfx, \bfx^*)
\end{equation} 
with $\scaleMathLine[0.93]{\bfb(\bfx, \bfx^*) = \nabla_{\frakq} \calH(\bfx) - \nabla_{\frakq_e} \calH_d(\bfx, \bfx^*) - \bfK_\bfd \nabla_{\frakp_e}\calH_d(\bfx, \bfx^*)}$ achieves the closed-loop dynamics:
\begin{equation}
\label{eq:sys_pch_des_Rn}
\begin{bmatrix}
\dot{\mathbf\frakq}_e \\
\dot{\mathbf\frakp}_e \\
\end{bmatrix}
= \begin{bmatrix}
\bfzero & \bfI_n \\
-\bfI_n & -\bfK_d
\end{bmatrix}
\begin{bmatrix}
\nabla_{\frakq_e} \calH_d(\bfx,\bfx^*) \\
\nabla_{\frakp_e} \calH_d(\bfx,\bfx^*)
\end{bmatrix} + \begin{bmatrix} \bfzero \\ \bfd \end{bmatrix},
\end{equation}
where $\bfd = \bfd_1 + \bfd_2$ as in \eqref{eq:total-disturbance} and $\bfd_2$ is as in \eqref{eq:disturbance_d_2}.

\begin{figure*}[t]
\centering
\begin{subfigure}
        \centering
        \includegraphics[height=29mm]{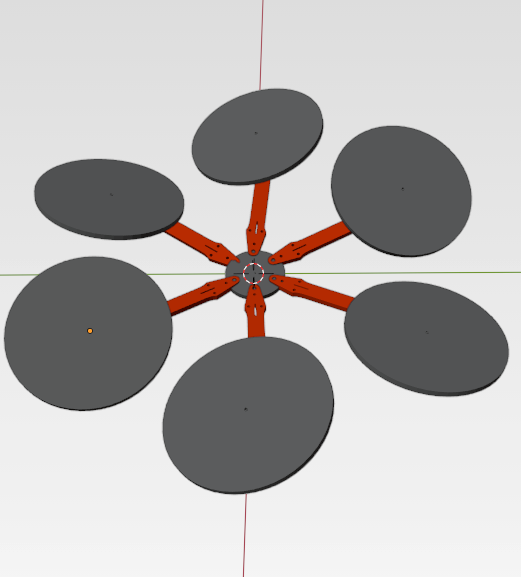}
        \label{fig:hexarotor}
\end{subfigure}%
\hfill
\begin{subfigure}
        \centering
        \includegraphics[height=30mm]{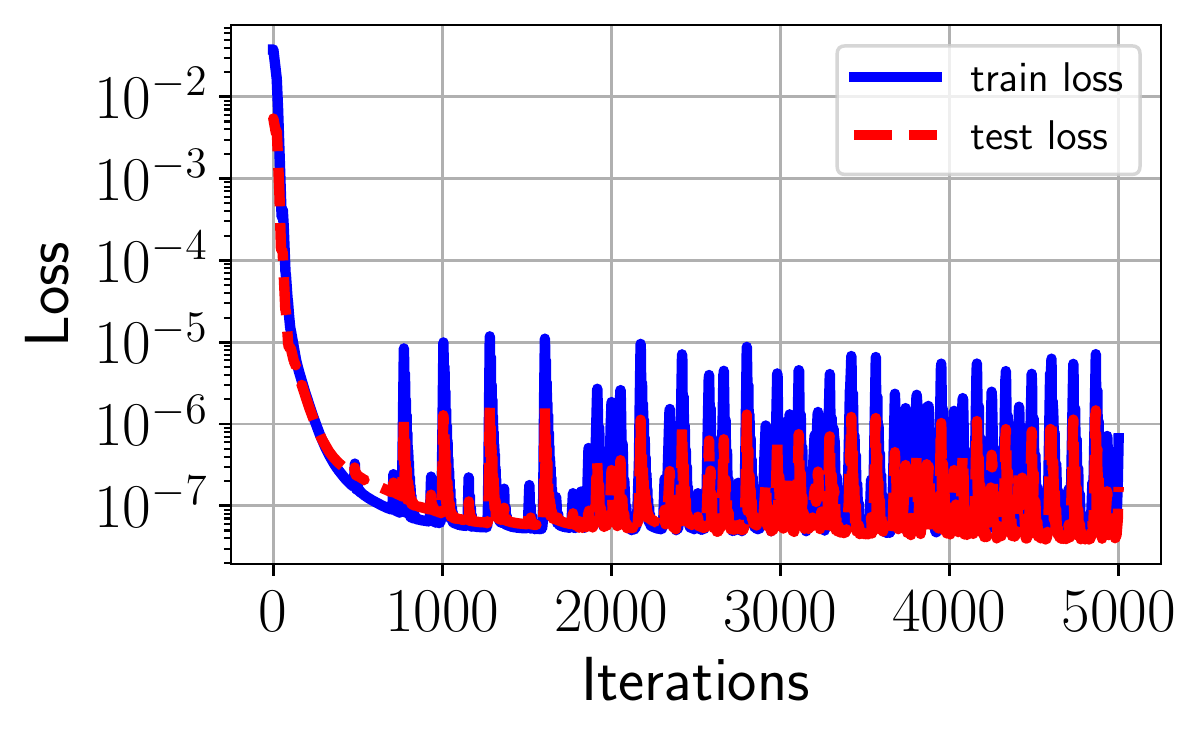}
        \label{fig:loss_log}
\end{subfigure}%
\hfill%
\begin{subfigure}
        \centering
        \includegraphics[height=31mm]{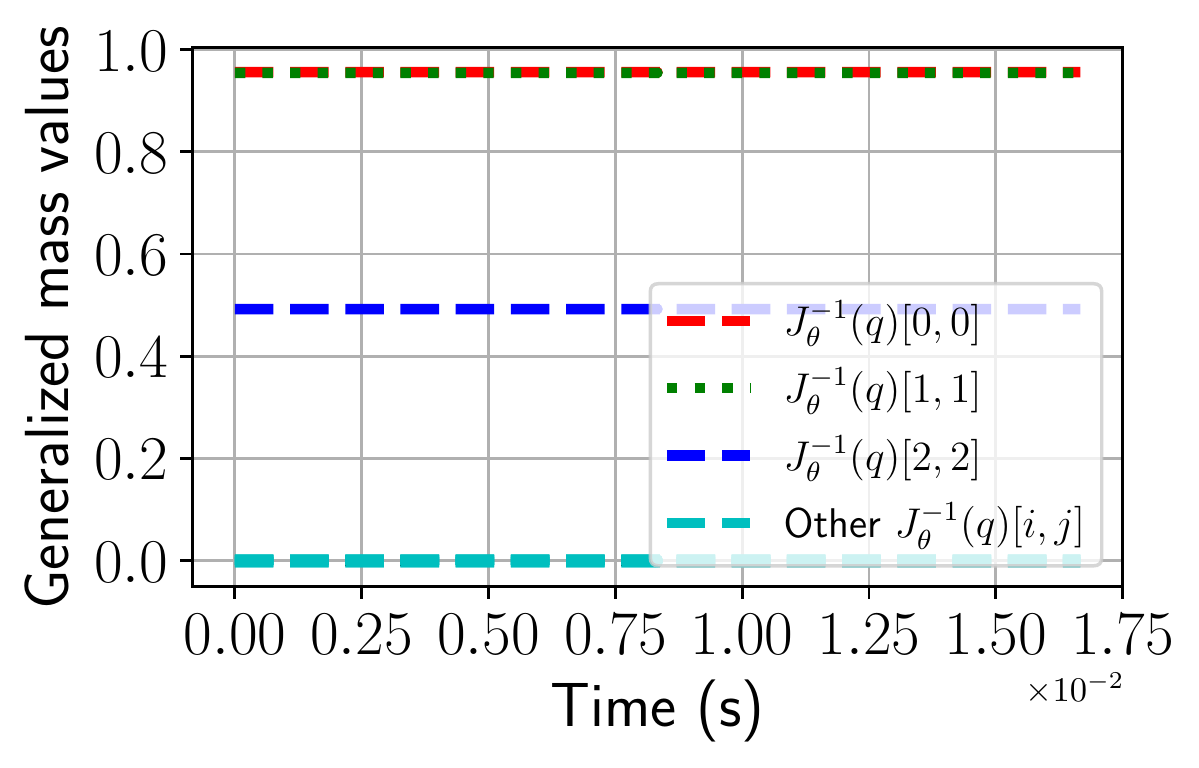}
        \label{fig:learned_inertia}
\end{subfigure}%
\hfill%
\begin{subfigure}
        \centering
        \includegraphics[height=31mm]{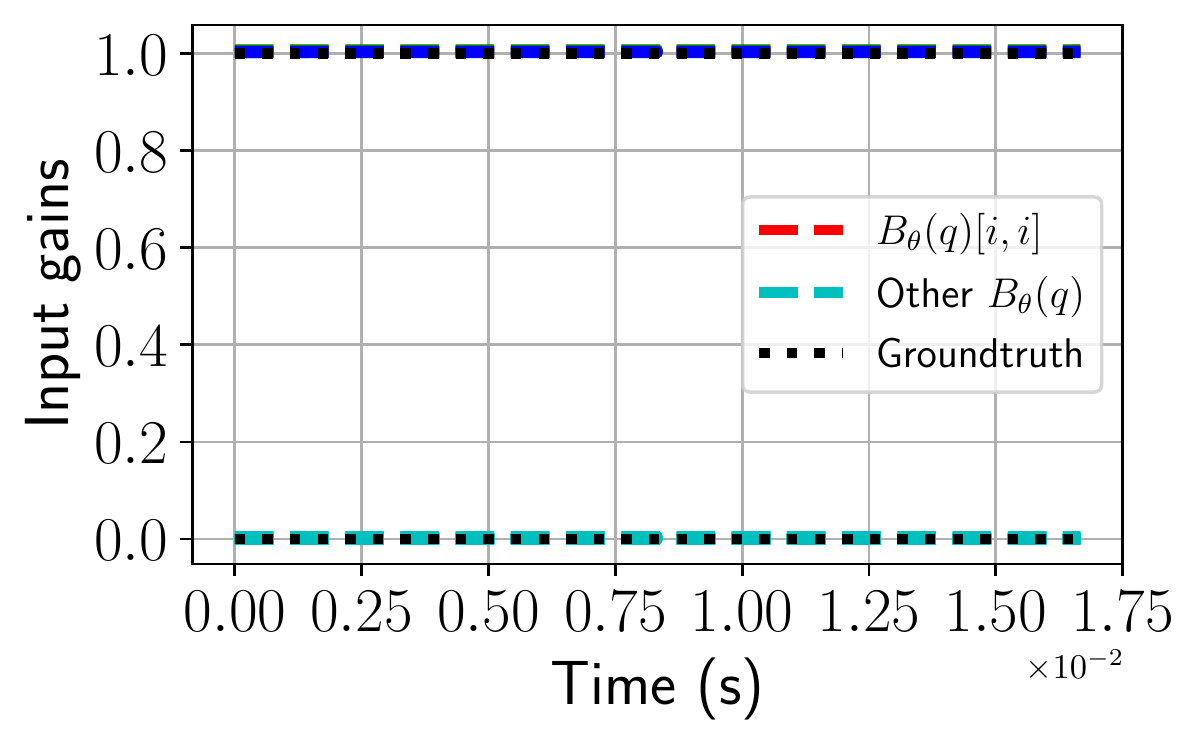}
        \label{fig:learned_g}
\end{subfigure}%
\caption{\NEWTD{$SE(3)$ Hamiltonian neural ODE network (left to right): (a) simulated hexarotor for evaluation, (b) training loss, (c) learned inverse inertia $\bfJ_\bftheta(\frakq)^{-1}$, and (d) learned input matrix $\bfB_\bftheta(\frak{q})$ along a test trajectory, evaluated on the simulated hexarotor.}}
\label{fig:fa_drone_learning}
\end{figure*}

\begin{theorem} \label{thm:robust_IDA_PBC_Rn} 
Consider the Hamiltonian system in \eqref{eq:sys_pch_Rn} with desired regulation point $\bfx^* = (\frakq^*, \mathbf{0})$ and control law in \eqref{eq:controller_IDA_PBC_Rn} with parameters $k_\bfp$, $\bfK_\bfd$. Under Assumptions \ref{asp:bounded_matched_d} \& \ref{asp:constant_M}, the function:
\begin{equation} \label{eq:ISS_lyap_Rn}
\calV(\bfx, \bfx^*) =  \calH_d(\bfx, \bfx^*) + \rho \frac{d}{dt} \calU_d(\frakq_e)
\end{equation}
with $\calU_d(\frakq_e) = \frac{k_{\bfp}}{2}\|\frakq_e\|^2$ is an ISS-Lyapunov function \cite{sontag2008input} with respect to $\bfd$ and satisfies:
\begin{equation} \label{eq:V_bounds_Rn} 
\begin{aligned} 
k_1 \norm{\bfz}^2 \leq \calV(\bfx, \bfx^*)  &\leq k_2 \norm{\bfz}^2, \\ 
\dot{\calV}(\bfx, \bfx^*) &\leq -k_3 \norm{\bfz}^2 + k_\gamma \delta_{\bfd}^2,
\end{aligned}
\end{equation}
where $\bfz \coloneqq [k_{\bfp}\|\frakq_e\|\; \|\frakp_e\|]^\top \in \bbR^2$, $k_\gamma =\frac{1}{2 \lambda_{\min}(\bfK_\bfd)} +\frac{\rho \lambda_2^2}{2 \lambda_1}$, $\lambda_1 \coloneqq \lambda_{\min}(\bfM^{-1})$, $\lambda_2 \coloneqq \lambda_{\max}(\bfM^{-1})$, $k_1 = \frac{1}{2} \lambda_{\min}(\bfQ_1)$, $k_2 = \frac{1}{2} \lambda_{\max}(\bfQ_2)$, $k_3 = \frac{1}{2}\lambda_{\min}(\bfQ_3)$, and the associated matrices $\bfQ_1$, $\bfQ_2$, $\bfQ_3$ are defined as:
\begin{equation} \label{eq:Q-matrices_Rn}
\begin{aligned}
\bfQ_1 &= \begin{bmatrix} k_{\bfp}^{-1}, & - \rho \lambda_2 \\
- \rho \lambda_2 & \lambda_1 \end{bmatrix}, \quad
\bfQ_2 = \begin{bmatrix} k_{\bfp}^{-1},  & \rho \lambda_2\\
\rho \lambda_2 & \lambda_2 \end{bmatrix}, \\
\bfQ_3 &= \begin{bmatrix} \rho \lambda_1 & -\rho \gamma_d \lambda_2^2 \\ 
 -\rho \gamma_d \lambda_2^2 & \gamma_d \lambda_1^2  - 2\rho \lambda_2^2 k_\bfp \end{bmatrix}. \quad 
\end{aligned}
\end{equation}
Any initial state $\bfx$ converges exponentially to $\calS_{c_1} = \crl{\bfx | \calV(\bfx,\bfx^*) \leq c_1}$ with $c_1 \coloneqq \frac{k_2 k_\gamma}{k_3} \delta_{\bfd}^2$ and remains within. The error trajectory $\frakq_e(t)$ is uniformly ultimately bounded:
\begin{equation}\label{eq:uub_Rn}
\lim_{t \rightarrow \infty} \norm{\frakq_e(t)}^2 \leq \frac{c_1}{k_1 k_\bfp^2} = \frac{k_2k_\gamma}{k_1k_3 k_{\bfp}^2} \delta_\bfd^2.
\end{equation}
\end{theorem}

The proof of Thm.~\ref{thm:robust_IDA_PBC_Rn} follows the same steps as that of Thm.~\ref{thm:robust_IDA_PBC}, and is omitted due to space limitations. In contrast to Thm.~\ref{thm:robust_IDA_PBC}, the result in Thm.~\ref{thm:robust_IDA_PBC_Rn} for $\bbR^n$ holds globally, i.e., the region of attraction is $\calA = \bbR^n \times \bbR^n$. Thus, the disturbance magnitude bound $\delta_{\bfd}$ can be arbitrarily large.

The safety analysis in Sec.~\ref{sec:ham_control}.\ref{sec:safety_analysis} can be modified with a new safety margin:
\begin{multline} \label{eq:def_deltaE_Rn}
 \Delta E(\bfx, \bfx^*) \coloneqq k_1 k_\bfp^2 \bar{d}{\,}^2 \!\prl{\frakq^*, \calO}  - \calV(\bfx, \bfx^*) \\
 + \max \crl{c_1 - \calV(\bfx, \bfx^*), 0},
\end{multline}
as Thm.~\ref{thm:robust_IDA_PBC_Rn} holds globally. The reference governor lifting function can be chosen as $\ell(\bfg) = [\bfg^\top\;\bf0^\top]^\top$. The governor state update remains the same as in \eqref{eq:gov_ctrl}. The robustness analysis extends the safe tracking results in \cite{Li_SafeControl_flatSys} and \cite{RG_Omur_ICRA17}.
}

\begin{figure*}[t]
\centering
\begin{subfigure}
        \centering
        \includegraphics[width=0.30\linewidth]{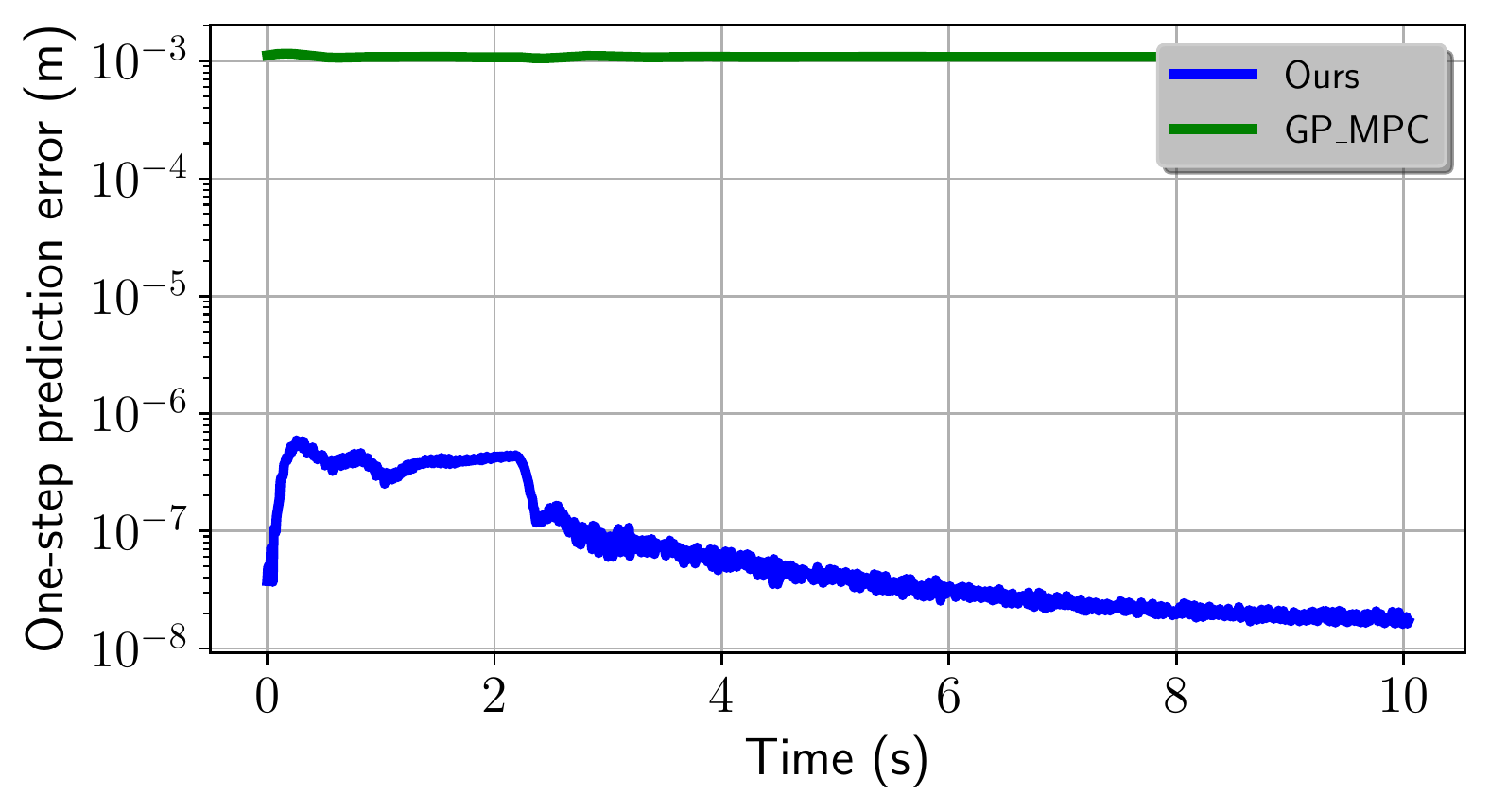}
        \label{fig:predictionerror_compare}
\end{subfigure}%
\hfill
\begin{subfigure}
        \centering
        \includegraphics[width=0.30\linewidth]{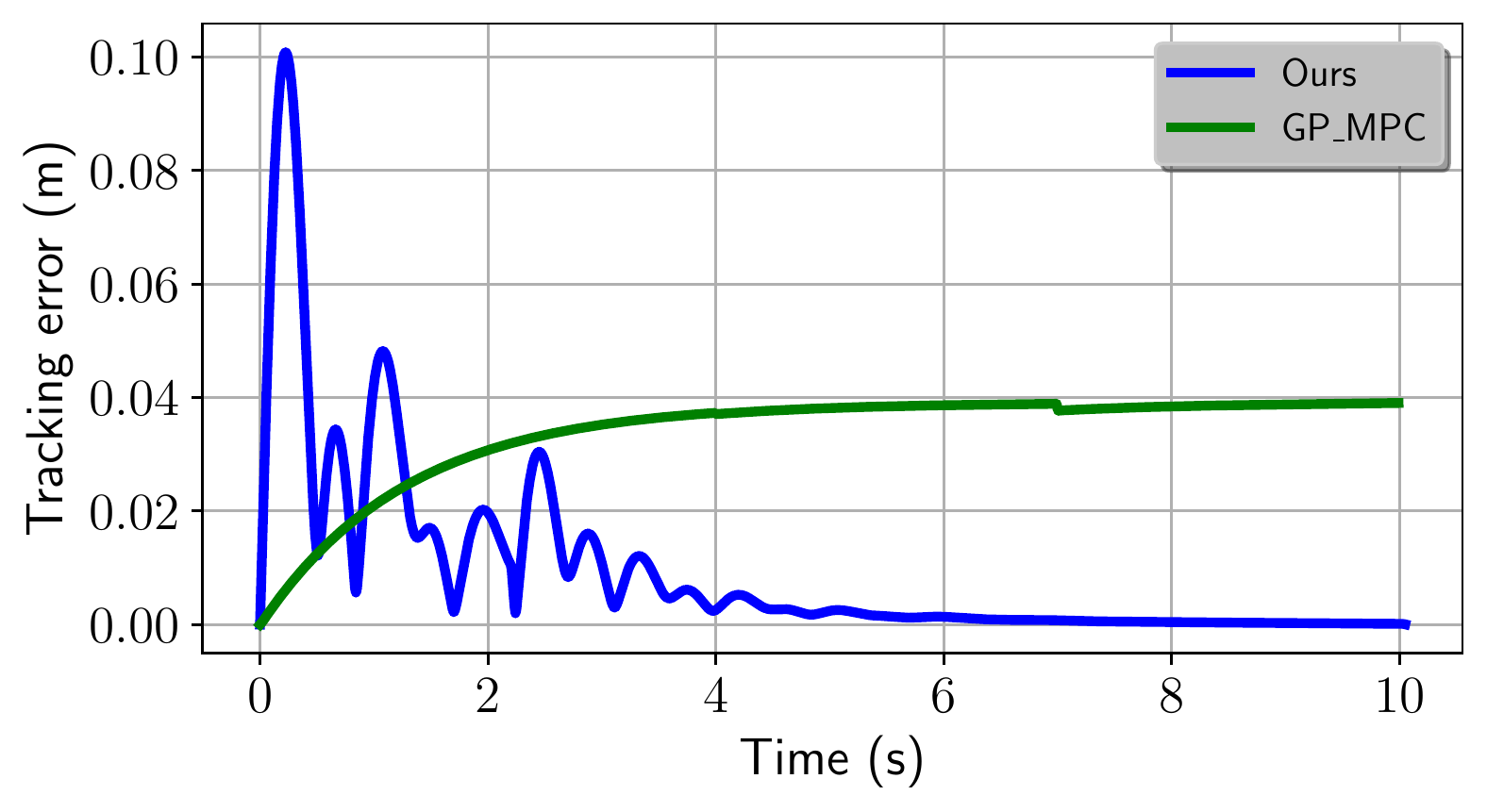}
        \label{fig:trackingerror_compare}
\end{subfigure}%
\hfill
\begin{subfigure}
        \centering
        \includegraphics[width=0.30\linewidth]{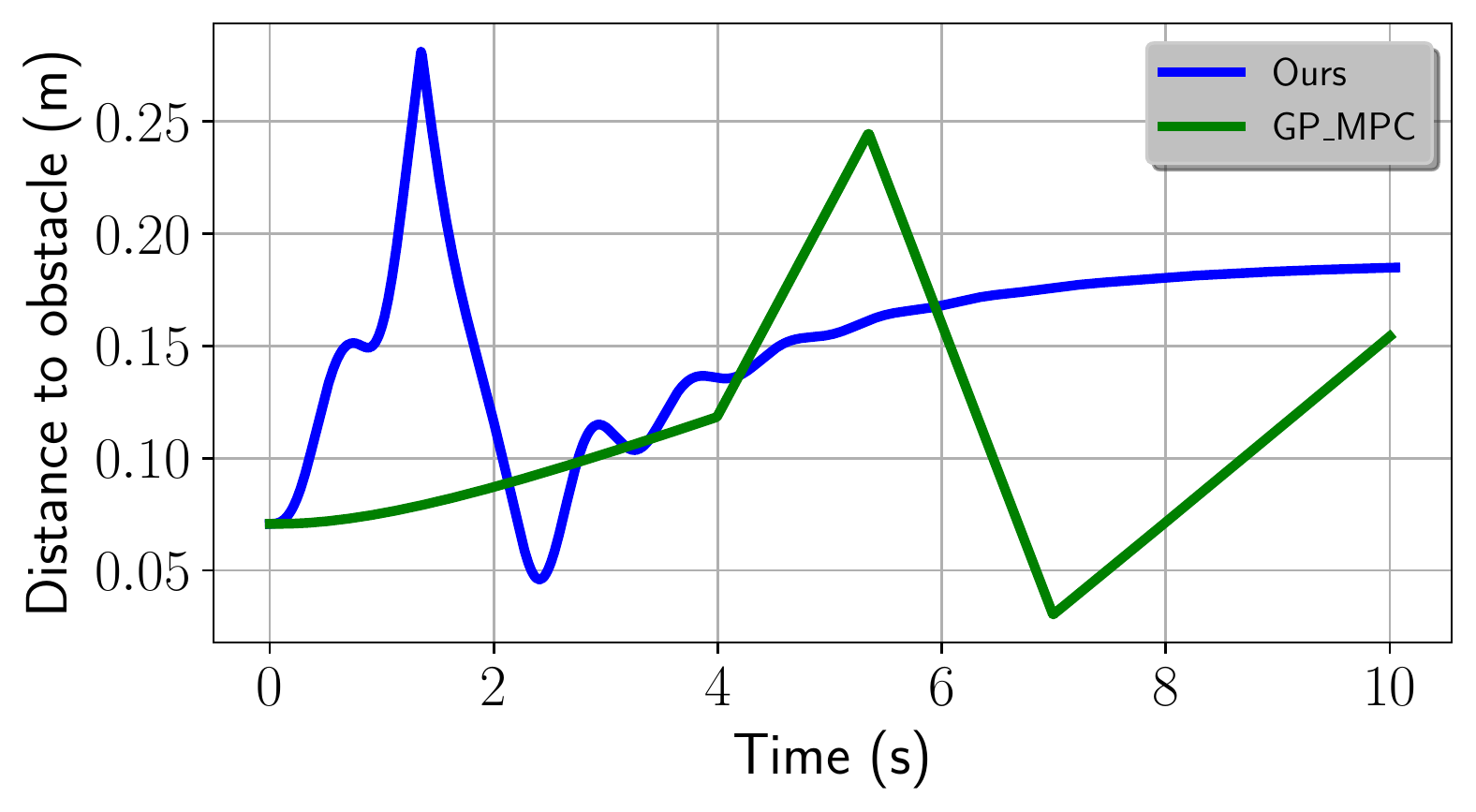}
        \label{fig:dist2obs_compare}
\end{subfigure}%
\caption{\NEWTD{Comparison of position prediction errors (left) between the learned Hamiltonian dynamics and GP model, tracking errors (middle) and distance to obstacles (right) between our safe tracking controller and GP-MPC.}}
\label{fig:gpmpc_compare_error}
\end{figure*}
\begin{figure}[t]
\centering
\begin{subfigure}
        \centering
        \includegraphics[width=0.49\linewidth]{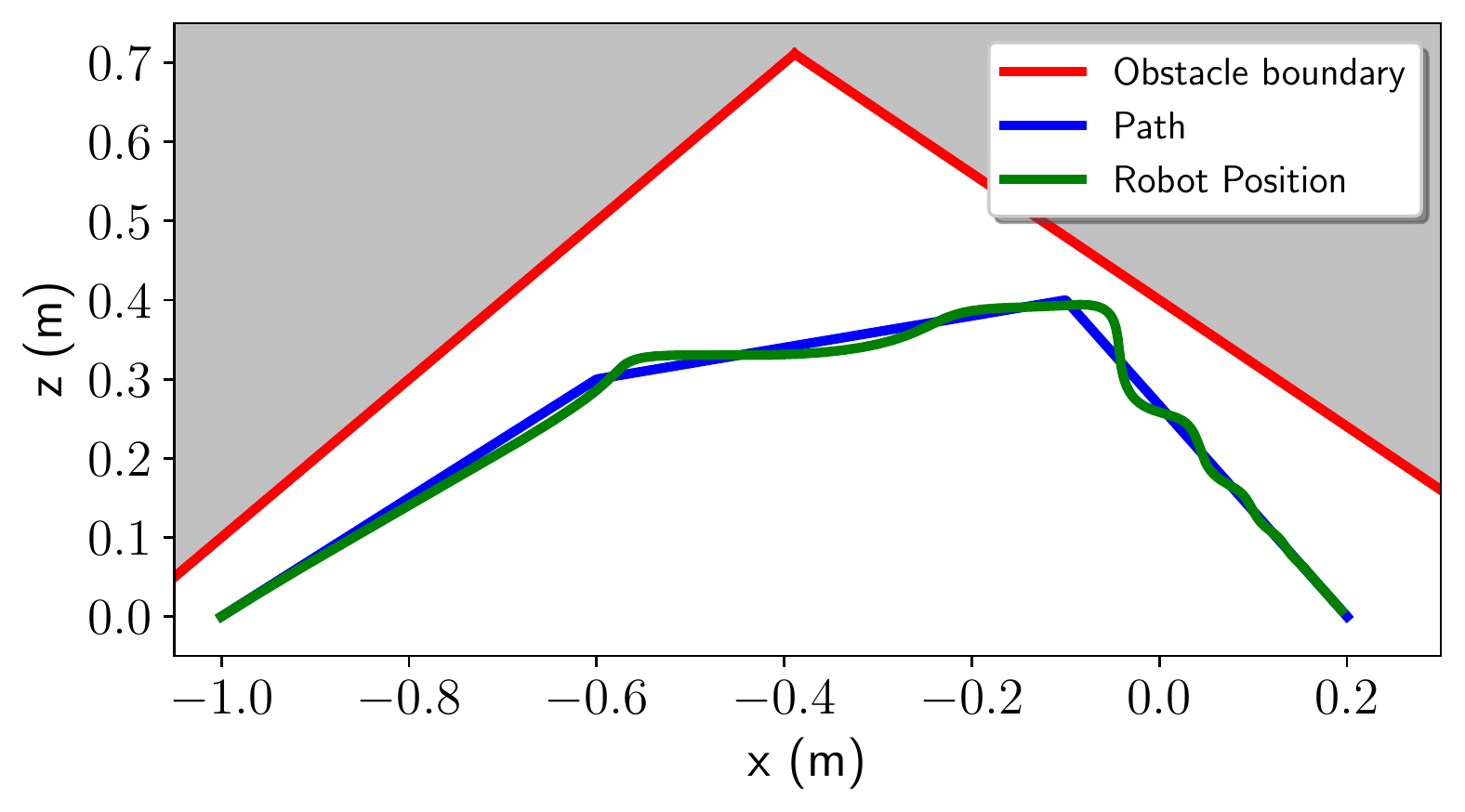}
        \label{fig:rgham_2d}
\end{subfigure}%
\begin{subfigure}
        \centering
        \includegraphics[width=0.49\linewidth]{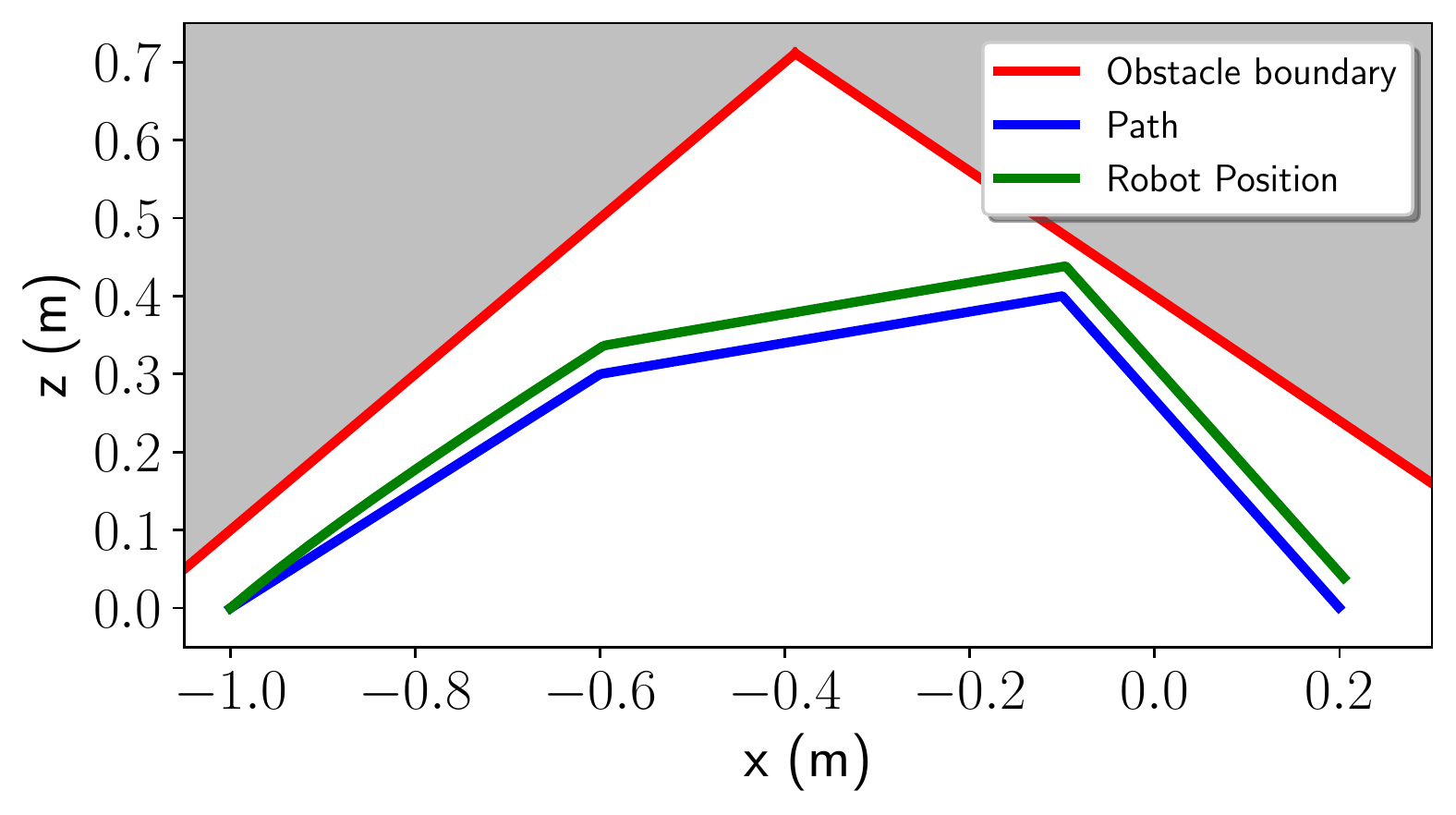}
        \label{fig:gpmpc_2d}
\end{subfigure}%
\caption{\NEWTD{Path tracking with our approach (left) and GP-MPC \cite{hewing2019cautious} (right).}}
\label{fig:gpmpc_compare}
\end{figure}

%% file: tex/evaluation.tex
\section{Evaluation}
\label{sec:evaluation}

\NEWZL{
We evaluate our robust and safe tracking controller using simulated hexarotor \NEWZL{and quadrotor} robots in 2D and 3D environments with ground-truth mass $m = 6.77$ kg, and inertia matrix $\bfJ = \text{diag}([1.05, 1.05, 2.05])$ 
kg$\,\cdot\,$m$^2$ 
, inspired by the solar-powered UAV in \cite{kingry2018design}. The robot's ground-truth dynamics satisfy Hamilton's equations \eqref{eq:Ham_sys_pf_nlin} with generalized mass $\bfM(\mathbf\frakq) = \diag(m\bfI, \bfJ)$, potential energy 
$\calU(\frakq) = mg\begin{bmatrix}
0 & 0 & 1
\end{bmatrix} \bfp$, 
where $\bfp$ is the position and $g\approx 9.8ms^{-2}$ is the gravitational acceleration. The input matrices for the hexarotor and the quadrotor are $\bfB(\frakq)= \bfI$ and $\bfB(\frakq) = \begin{bmatrix} \bf0_{4\times 2} & \bfI_{4\times 4} \end{bmatrix}^\top$, respectively. The control input $\bfu$ of the hexarotor includes a 3D force and a 3D torque while that of the quadrotor includes a scalar force and a 3D torque.

For all experiments, the following control gains are used for our controller in Sec. \ref{sec:ham_control}.\ref{subsec:ham_control}: $k_\bfp = 20$, $k_\bfR = 50$, $\bfK_d = 15\bfI$ in \eqref{eq:desired_hamiltonian}. The parameters shown in Thm.~\ref{thm:robust_IDA_PBC} are: $\alpha=2$, $\beta=20$, $c_1 = 2.2050$, $c_2 = 8.8200$, $\rho = 3.5822 \times 10^{-5}$. The control gain for the governor in \eqref{eq:gov_ctrl} is $k_g = 0.5$. 
The control loop frequency for all experiments is at $120$ Hz.
}

\NEWZL{While, our evaluation focuses on rotorcraft aerial robots, the methodology for system identification and control synthesis proposed in this paper is general. The exact same approach is applied to hexarotor, quadrotor and other ground and marine vehicles. This is in contrast with other system identification and control synthesis methods, which require knowledge of the dynamics structure, careful experiment design, and domain expertise for the particular system.}

\subsection{Evaluation of $SE(3)$ Hamiltonian dynamics learning}
\label{subsec:ham_dyn_learning_evaluation}

We consider a simulated hexarotor unmanned aerial vehicle (UAV) (Fig. \ref{fig:fa_drone_learning}) with fixed-tilt rotors pointing in different directions \cite{hexarotor} and \Revised{a simulated quadrotor UAV}.
%
Since the mass $m$ of the UAVs can be easily measured, we assume the mass $m$ is known,
leading to a known potential energy $\calU(\frakq) = mg\begin{bmatrix}
0 & 0 & 1
\end{bmatrix} \bfp$. 
We approximate the inverse generalized mass matrix by $\bfM^{-1}_\bftheta(\mathbf\frakq) =\diag(m^{-1}\bfI, \bfJ_{\bftheta}^{-1}(\mathbf\frakq))$ and learn $\bfJ_{\bftheta}(\mathbf\frakq)^{-1}$ and $\bfB_\bftheta(\frakq)$ from data. 

We mimic manual flights in an area free of obstacles using a PID controller and drive the UAVs from a random initial pose to $18$ desired poses, generating $18$ $1$-second trajectories. We shift the trajectories to start from the origin and create a dataset $\mathcal{D} = \{t_{0:N}^{(i)},\mathbf\frakq_{0:N}^{(i)}, \bfzeta_{0:N}^{(i)}, \bfu^{(i)}_{0:N-1})\}_{i=1}^D$ with $N = 24$ and $D = 18$. The Hamiltonian-based neural ODE network is trained with the dataset $\calD$, as described in Sec. \ref{subsec:ham_dyn_learning}, for 5000 iterations and learning rate $10^{-4}$. For the hexarotor, Fig. \ref{fig:fa_drone_learning}(c) shows the loss function during training. Note that if we scale $\bfM_\bftheta(\frakq)$ and the input matrix $\bfB(\frakq)$ by a constant $\gamma$, the dynamics of $(\frakq, \bfzeta)$ in \eqref{eq:Ham_sys_pf_nlin} and \eqref{eq:hamiltonian_zetadot} does not change. Fig. \ref{fig:fa_drone_learning}(d) and \ref{fig:fa_drone_learning}(e) plot the scaled version of the learned inverse mass $\bfJ_\bftheta(\frakq)^{-1}$ and the input matrix $\bfB_\bftheta(\frakq)$, converging to the constant ground truth values. \NEWZL{We achieve similar results for the quadrotor using the same training process.}


\subsection{Evaluation of robust safe tracking control of a learned 2D hexarotor Hamiltonian model}

Next, we compare our approach with a GP-MPC technique \cite{hewing2019cautious} using a simulated 2D fully-actuated hexarotor UAV, moving on the xz-plane with position $\bfp = \begin{bmatrix} x, 0, z \end{bmatrix}$ and orientation $\bfR = \bfR_\psi$ determined by the pitch angle $\psi$. The control input is a 3D wrench, including a 2D force and a 1D torque. 

As we only consider the pitch angle $\psi$, we are interested in the inertia value $J^{yy}$ and ignore $J^{xx}$ and $J^{zz}$. We assume that the generalized mass  $m$ and $J^{yy}$ are unknown for the 2D hexarotor and approximated by $m_\bftheta$ and $J^{yy}_\bftheta$, respectively. 
The input gain $\bfB(\mathbf\frakq)$ is assumed known. 

Let $m_0 = 1.5m$ and $J^{yy}_0 = 1.5J^{yy}$ be initial guesses of the mass $m$ and the inertia $J^{yy}$. We model the approximated mass inverse $m^{-1}_\bftheta$ and inertia inverse ${J^{yy}_\bftheta}^{-1}$ as: 

\begin{equation*}
\scaleMathLine{
    m^{-1}_\bftheta = \prl{\sqrt{m_0^{-1}} + L_1(\frakq;\bftheta)}^2,
    {J^{yy}_\bftheta}^{-1} = \prl{\sqrt{{J^{yy}_0}^{-1}} + L_2(\frakq;\bftheta)}^2,}
\end{equation*}
where $L_1(\frakq;\bftheta)$ and $L_2(\frakq;\bftheta)$ are two neural networks, representing the residual mass inverse and inertia inverse to be learned. In GP-MPC \cite{hewing2019cautious}, the dynamics \eqref{eq:Ham_sys_pf_nlin} are split into a prior nominal model with the prior mass $m_0$ and~inertia~$J^{yy}_0$, and residual dynamics, modeled by a GP regression model.

To collect training data, we place the simulated hexarotor at an initial location $(x,z) = (-1,0)$ and apply random control inputs to obtain $\mathcal{D} = \{t_{0:1}^{(i)},\mathbf\frakq_{0:1}^{(i)}, \bfzeta_{0:1}^{(i)}, \bfu^{(i)}_0\}_{i=1}^{150}$. Our Hamiltonian neural ODE network is trained with the dataset $\calD$, as described in Sec.~\ref{subsec:ham_dyn_learning}. For GP-MPC, the same dataset $\mathcal{D}$ is used to train a GP regression model of the residual dynamics as described in \cite{hewing2019cautious} and implemented in \cite{brunke2021safe}.

We assume there are two walls in the environments, generating two safety constraints on the robot position: $-x + z < 1.1,\; 0.8x + z < 0.4$.
The task is to track a predefined piecewise linear path $\bfr$, shown in Fig. \ref{fig:gpmpc_compare}, while safely avoiding collision with the walls. We adapt the GP-MPC implementation by \cite{brunke2021safe} for the 2D hexarotor and enforce the safety constraints probabilistically with $95\%$ confidence interval using the GP model uncertainty. To propagate the model uncertainty through \NEWTD{a horizon of $10$ time steps}, we linearize the dynamics model around the hovering state and propagate the state mean and covariance using the mean equivalence technique \cite{hewing2019cautious,brunke2021safe} \NEWTD{with a time step of $1/120$ s}. Meanwhile, our learned Hamiltonian neural ODE model is used with the safe tracking controller described in Sec. \ref{sec:tracking} to perform the task and enforce safety constraints.

Fig. \ref{fig:gpmpc_compare_error} compares the prediction errors of our learned neural ODE network and the GP model. We collect the robot states and control inputs, generated by our controller while tracking the path, and predict the next state. Fig.~\ref{fig:gpmpc_compare_error}~(left)~plots~the prediction error over time, showing that we achieve better prediction than the trained GP model. This reflects the difference between our model, which encodes the Hamiltonian structure and translation equivariance in the network architecture, and the GP model, which incurs higher model uncertainty in locations far from the data points.

Fig. \ref{fig:gpmpc_compare_error} and \ref{fig:gpmpc_compare} show tracking performance of our approach and GP-MPC. We compare the tracking error of both methods, calculated as the distance from the robot position to the reference point, specified by the governor in our approach and by time parameterization of the path in GP-MPC: $\bfp^*(t) = \bfr(\min(t,10)/10)$, 
%
%
i.e., the GP-MPC method finishes the task in about $10$ seconds, similar to the tracking time of our approach. Our controller is able to track the path more accurately than GP-MPC, illustrated qualitatively in Fig. \ref{fig:gpmpc_compare} and quantitatively in Fig. \ref{fig:gpmpc_compare_error} (middle). This can be explained by the higher predictions errors shown in Fig. \ref{fig:gpmpc_compare_error} (left), which grow quickly after multiple time steps due to uncertainty propagation. Both our safe tracking controller with learned Hamiltonian dynamics and the GP-MPC safe controller keep the hexarotor in the safe region, i.e., the distance to the obstacles is always positive in Fig. \ref{fig:gpmpc_compare_error} (right).

\begin{figure*}[t]
\begin{minipage}{0.45\textwidth}
\centering
        \includegraphics[height=45mm,trim=10mm 10mm 10mm 10mm, clip]{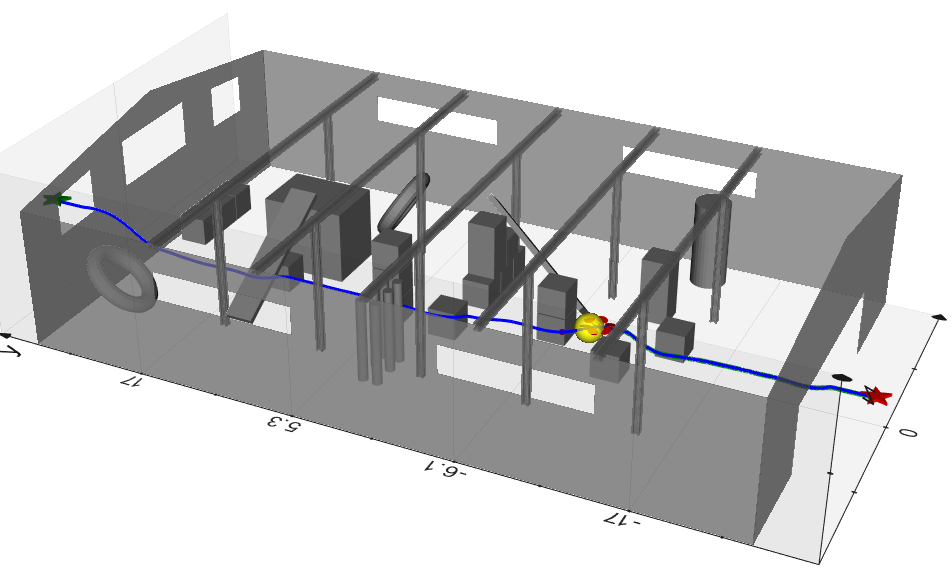}
\end{minipage}
\hfill
\begin{minipage}{0.45\textwidth}
\centering
         \includegraphics[width=0.93\linewidth]{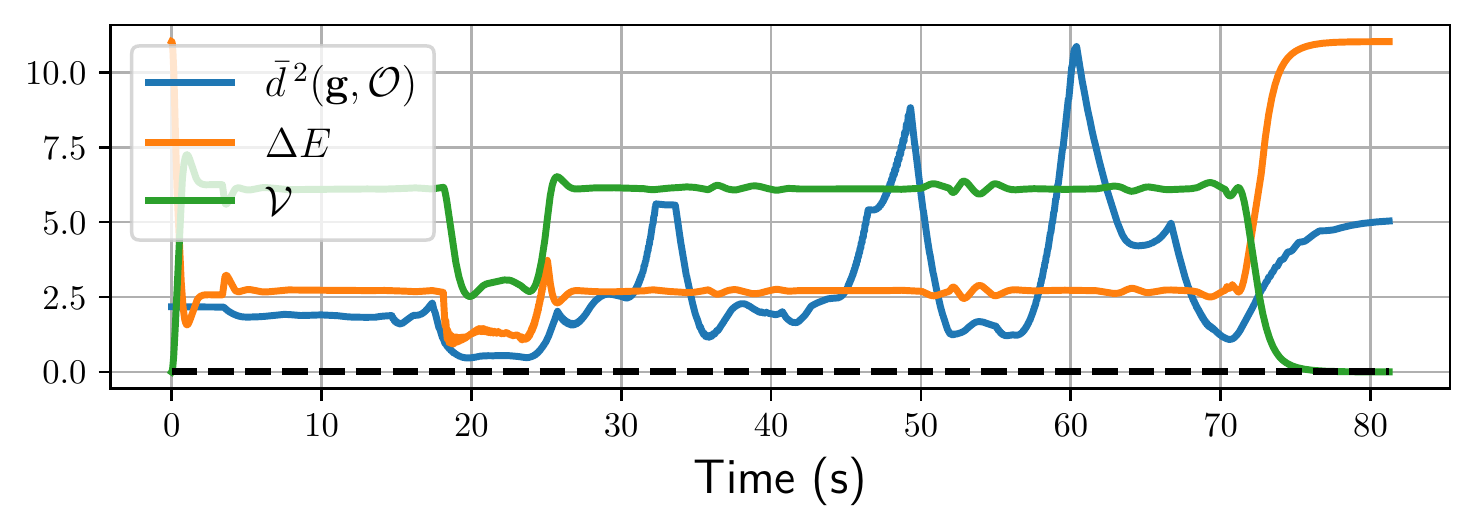}
         \includegraphics[width=0.95\linewidth]{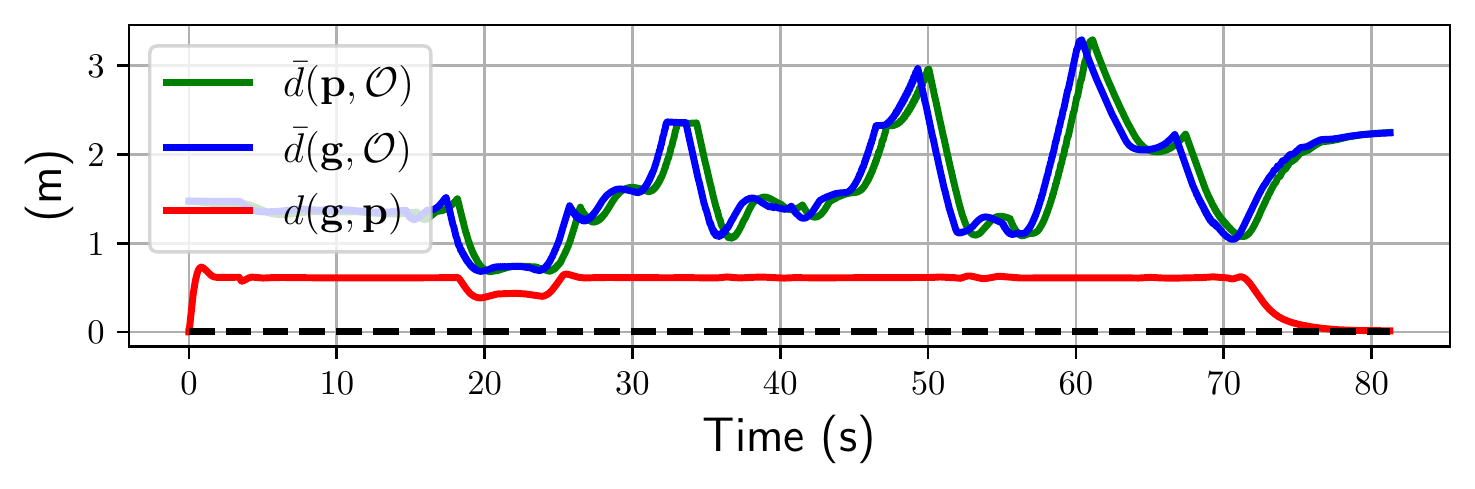}
\end{minipage}
\caption{\NEWZL{Safe navigation of a hexarotor system using learned model in a warehouse (left). The hexarotor (red body) navigates from a start (red star) to a goal location (green star) while avoiding obstacles. The obstacles are sensed by a simulated LiDAR sensor. The reference path, the robot path are shown in blue and green, respectively. Local safe set is shown in yellow sphere.
The right plots show the dynamic safety margin $\Delta E$, the Lyapunov function $\calV$, and the distance to the obstacles $\bar{d}(\bfp(t), \calO)$, indicating that the safety constraints are never violated.}}
\label{fig:warehouse3d_sim}
\end{figure*}

\begin{figure}[t]
\centering
        \centering
        \includegraphics[width=0.95\linewidth]{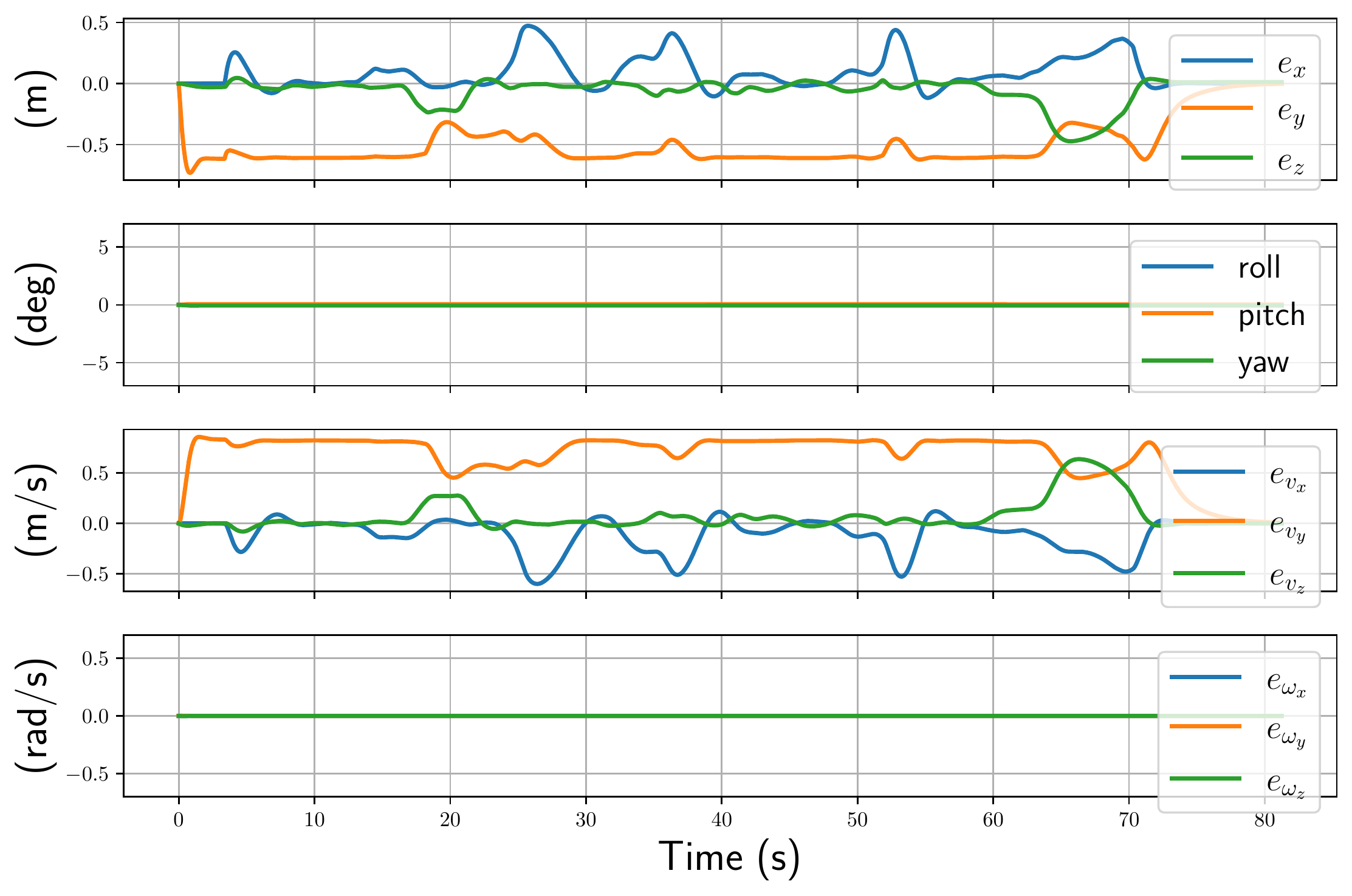}
\caption{\Revised{Tracking error of a hexarotor system (top to bottom): position, velocity, angle and angular velocity errors.}}
\label{fig:warehouse3d_tracking_err}
\end{figure}

\begin{figure}[t]
\centering
        \centering
        \includegraphics[width=0.95\linewidth]{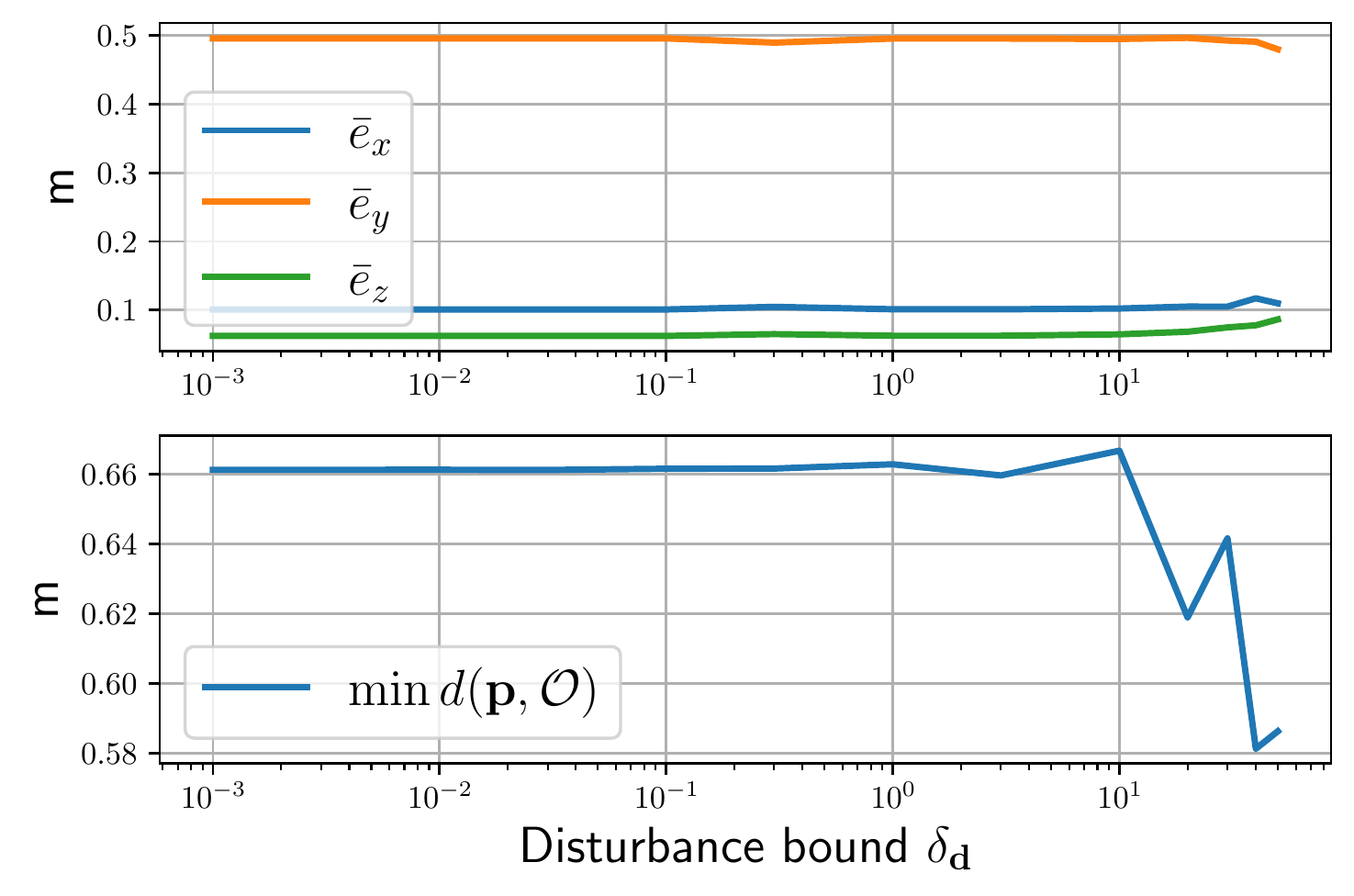}
\caption{\NEWZL{Tracking controller performance for hexarotor in warehouse simulation with the ground truth model subject to a disturbance  $\bfd$ with different magnitudes: the average position tracking error (top) and the minimum distance to obstacle (bottom).}}
\label{fig:robustness}
\end{figure}

\subsection{Evaluation of robust safe tracking control of a learned 3D fully-actuated hexarotor Hamiltonian model}
\label{subsec:eval_robustness_hexarotor}

This section evaluates our Hamiltonian dynamics learning and safe tracking control techniques using a simulated hexarotor UAV in a 3D environment. The task is to navigate from a start position to a goal in a cluttered warehouse environment without colliding with the obstacles $\calO$. The same control gains are used for this 3D navigation task as in the previous section.
A simulated LiDAR scanner provides point cloud measurements $\calP(t)$ of the surface of the unsafe set $\calO$, depending on the system pose at time $t$, with a maximum sensing range of $d_{max} = 30$ m. The distance from the governor $\bfg(t)$ to the unsafe set $\calO$ is approximated via $\bar{d}(\bfg(t); \calO) \approx \min_{\bfy \in \calP(t)} \norm{\bfg(t) - \bfy}$. \NEWZL{The reference path $\bfr$ is pre-computed using an A* planner and tracked in $\sim 80$ s.}

Fig.~\ref{fig:warehouse3d_sim} shows the behavior of the closed-loop hexarotor system in the warehouse environment. The safety margin $\Delta E(\bfx,\bfx^*)$ fluctuates during the tracking process but, as can be seen in Fig.~\ref{fig:warehouse3d_sim}, it never becomes negative. The augmented system $(\bfx, \bfg)$ is controlled adaptively, slowing down when the dynamic safety margin decreases (e.g., when the hexarotor is close to an obstacle or has large Lyapunov value $\calV$) and speeding up otherwise (e.g., when the robot is far away from the obstacles or has small total energy $\calV$). The simulations show that our control policy successfully drives the system from the start to the end of the reference path while avoiding sensed obstacles online, i.e., $d(\bfp, \calO)$ remains positive throughout the tracking task. 
Fig.~\ref{fig:warehouse3d_tracking_err} plots the tracking errors between the robot state $\bfx$ and the reference state $\bfx^*$ generated by the governor, showing that our controller tracks the path well. The tracking errors for the Euler angles and angular velocity, are close to $0$. The position and linear velocity errors in the $x$ and $z$ directions are close to zero as well while the errors in $y$ direction fluctuates around $-0.5$ m and $0.8$ m/s, respectively, and converges to $0$ at the end. This is expected as the robot stays behind the reference point, mostly in $y$ direction, and converges to the end of the path.

To evaluate the robustness of our controller, we repeat the warehouse experiment using the ground-truth dynamics, subject to a artificially generated disturbances $\bfd \in \bbR^6$ with different upper bounds $\delta_\bfd$. Each component of the disturbance $\bfd \in \bbR^6$ is uniformly generated in $\brl{-0.5\delta_\bfd, 0.5\delta_\bfd}$. If $\norm{\bfd} > \delta_\bfd$, we normalize the disturbance as $\delta_\bfd\bfd/\norm{\bfd}$. Our robust tracking controller successfully finishes the tracking task across a wide range of $\delta_\bfd$: $\brl{0.001, 0.01, 0.1, 1, 10, 20, 30}$. Larger $\delta_\bfd$ are not reported due to violation of the positiveness requirement on $\Delta E$. Fig.~\ref{fig:robustness} shows the average position errors and the minimum distance to obstacle during the tracking task versus the disturbance upper bound $\delta_\bfd$. The average position tracking errors remain similar against $\delta_\bfd$. The minimum distance to obstacle $d(\bfp, \calO)$ is always positive, illustrating the safety guarantees of our controller. This number starts decreasing when $\delta_\bfd > 1$ as larger disturbances can suddenly move the robot towards the obstacles.

\subsection{Evaluation of robust safe tracking control of a learned 3D underactuated quadrotor Hamiltonian model}
\label{subsec:quadrotor-sim}

\NEWTD{
In this section, we repeat the task of safely navigating from a start position to a goal in the same cluttered warehouse environment in Sec. \ref{sec:evaluation}.\ref{subsec:eval_robustness_hexarotor} with a quadrotor, whose model is learned from data as described in Sec. \ref{sec:evaluation}.\ref{subsec:ham_dyn_learning_evaluation}. As mentioned in Sec. \ref{sec:ham_control}, the control input in \eqref{eq:controller_IDA_PBC} would not introduce additional disturbance $\bfd_2$ when the matching condition \eqref{eq:matching_condition} is satisfied. For quadrotor, a maximal-rank left annihilator of the ground-truth $\bfB(\frakq) = \begin{bmatrix} \bf0_{4\times 2} & \bfI_{4\times 4} \end{bmatrix}^\top$ is $ \bfB^{\dagger}(\frakq) = \begin{bmatrix} \bfI_{2\times 2} & \bf0_{2\times 4} \end{bmatrix}$. The matching condition \eqref{eq:matching_condition} is satisfied if and only if the first two elements of $\bfb(\bfx,\bfx^*) = \begin{bmatrix} \bfb_\bfv^\top & \bfb_\bfomega^\top\end{bmatrix}^\top, \bfb_\bfv\in \bbR^3, \bfb_\bfomega\in \bbR^3$ in \eqref{eq:desired_control_IDA_PBC} equal to $0$, i.e. the force component $\bfb_\bfv$ coincides with the $z$-axis of the body frame. As guaranteeing this condition is hard, we instead use the force component in the world frame $\bfR \bfb_\bfv$ and a desired yaw angle $\psi^*$ to determine the desired rotation matrix, similar to \cite{lee2010geometric}. The vector $\bfR \bfb_\bfv$ is set as the $z$-axis of the desired frame, i.e., the third column $\bfb_3^*$ of the rotation matrix $\bfR^*$, to minimize the disturbance $\bfd_2$ in \eqref{eq:disturbance_d_2} from the matching condition. We calculate the second column $\bfb^*_2$ by projecting the second column of the yaw's rotation matrix $\bfb_2^\psi = [-\cos{\psi}, \sin{\psi}, 0]$ onto the plane perpendicular to $\bfb_3^*$. We use the controller \eqref{eq:controller_IDA_PBC} with $\bfR^* = [\bfb_1^* \quad \bfb_2^* \quad \bfb_3^*]$ where:
%
\begin{equation}
\bfb_3^* = \frac{\bfR\bfb_\bfv}{\Vert \bfR\bfb_\bfv \Vert}, \bfb_1^* = \frac{\bfb_2^\psi \times \bfb^*_3}{\Vert  \bfb_2^\psi \times \bfb^*_3 \Vert}, \bfb_2^* = \bfb_3^* \times \bfb_1^*,
\end{equation}
%
and $\hat{\bfomega}^* = \bfR^{*\top}\dot{\bfR}^*$ for our tracking task.

\begin{figure}[t]
\centering
\begin{subfigure}
        \centering
         \includegraphics[width=0.93\linewidth]{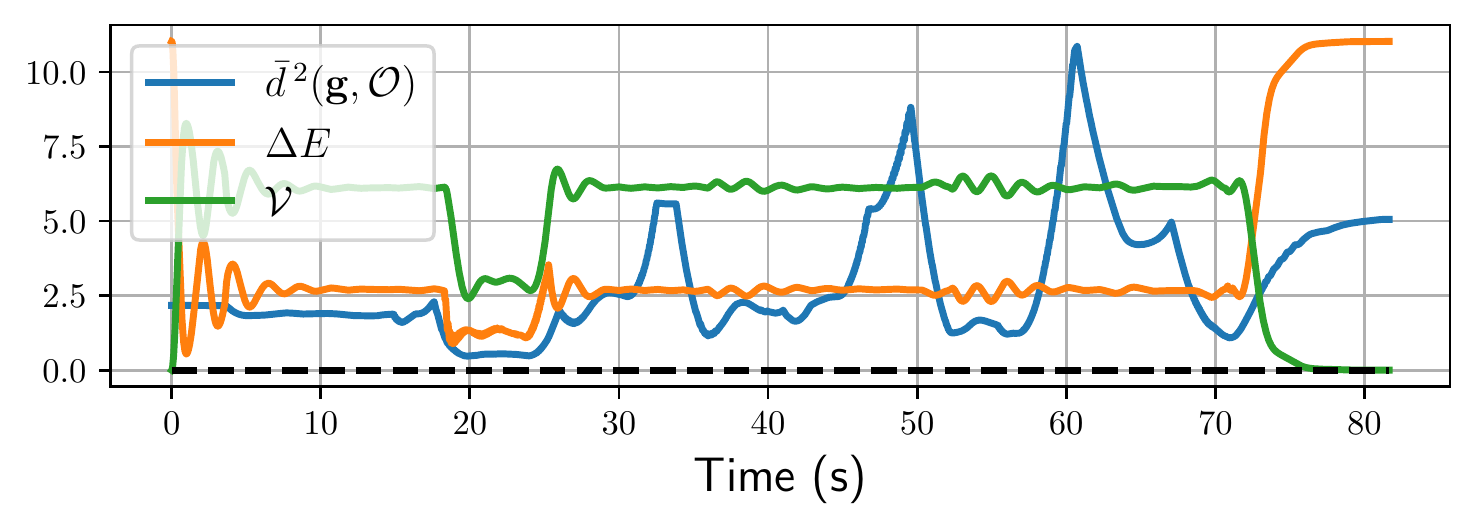}
\end{subfigure}%
\hfill
\begin{subfigure}
        \centering
        \includegraphics[width=0.95\linewidth]{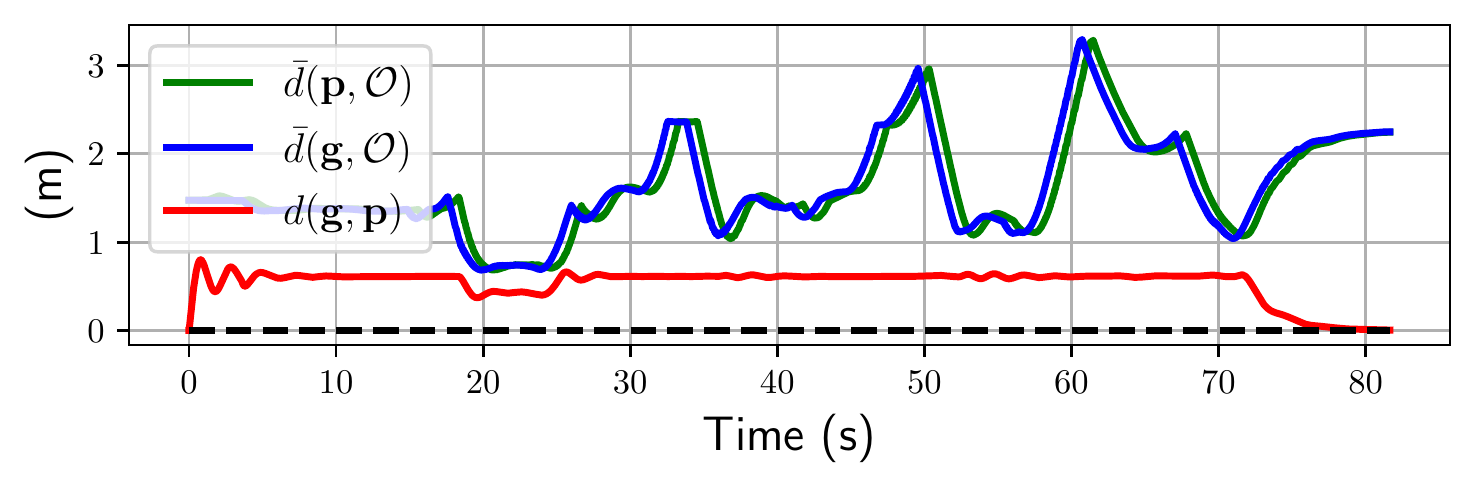}
\end{subfigure}%
\caption{\NEWTD{Safe navigation of quadrotor system (learned model) in a warehouse: the dynamic safety margin $\Delta E$, the Lyapunov function $\calV$ (top) and the distance to the obstacles $\bar{d}(\bfp(t), \calO)$ (bottom), indicating that the safety constraints are never violated.}}
\label{fig:warehouse3d_sim_quad}
\end{figure}

\begin{figure}[t]
\centering
        \centering
        \includegraphics[width=0.95\linewidth]{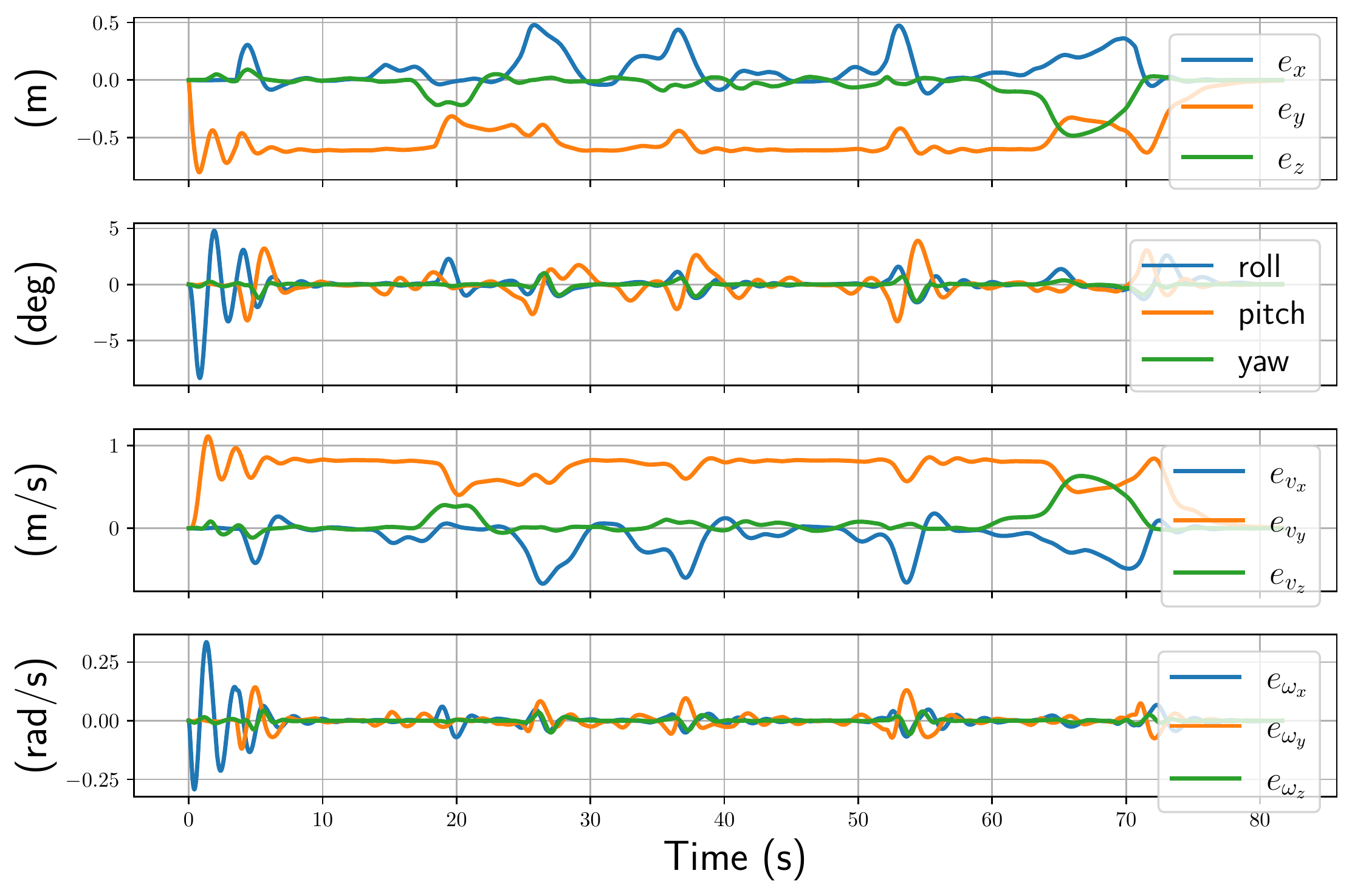}
\caption{\Revised{Tracking error of a quadrotor system (top to bottom): position, velocity, angle and angular velocity errors.}}
\label{fig:warehouse3d_tracking_err_quad}
\end{figure}

We successfully finish the task with the quadrotor while remaining safe for the entire experiment, as shown in Fig.~\ref{fig:warehouse3d_sim_quad}, with similar behavior of the closed-loop quadrotor system in terms of the safety margin, Lyapunov function and distance to obstacle compared to Sec. \ref{sec:evaluation}.\ref{subsec:eval_robustness_hexarotor}. However, the orientation tracking error of quadrotor (Fig.~\ref{fig:warehouse3d_tracking_err_quad}) is larger than that of hexarotor, as expected since the quadrotor is underactuated. }

\Revised{
\subsection{Evaluation of our approach against unmodeled noise}
In this section, we verify the robustness of our controller against unmodeled noise on a simulated hexarotor by injecting high frequency noise (e.g., propeller vibration) into control inputs and simulating state estimation errors. In particular, a $4.8$ kHz 6D sinusoidal signal with amplitude $5$ is generated for high frequency noise. 
Meanwhile, state estimation errors in positions, Euler angles, linear and angular velocity are randomly generated with zero mean and standard deviation, chosen from \cite{hoang2021noise} (position: $0.01$ m, Euler angle: $0.01$ degree, linear velocity: $0.02$ m/s and angular velocity: $0.14$ degree/s). We consider the task of stabilizing to a static governor, i.e. the governor is not moving, with the learned dynamics model: without any unmodeled noise (base), with high-frequency noise, and with state estimation error. Fig. \ref{fig:robust_hex} plots the Lyapunov function $\calV$ and the safety margin $\Delta E$ over time. Our controller is not affected significantly from the high-frequency noise, potentially because the noise's effect is canceled out due to its zero mean. Our controller is safe against the state estimation errors from \cite{hoang2021noise}, i.e. $\Delta E > 0$ over time, but fails to remains safe, i.e. $\Delta E < 0$ at some times, if we triple the noise deviation.
}

\label{subsec:unmodeled_noise}
\begin{figure}[t]
\centering
        \centering
        \includegraphics[width=0.95\linewidth]{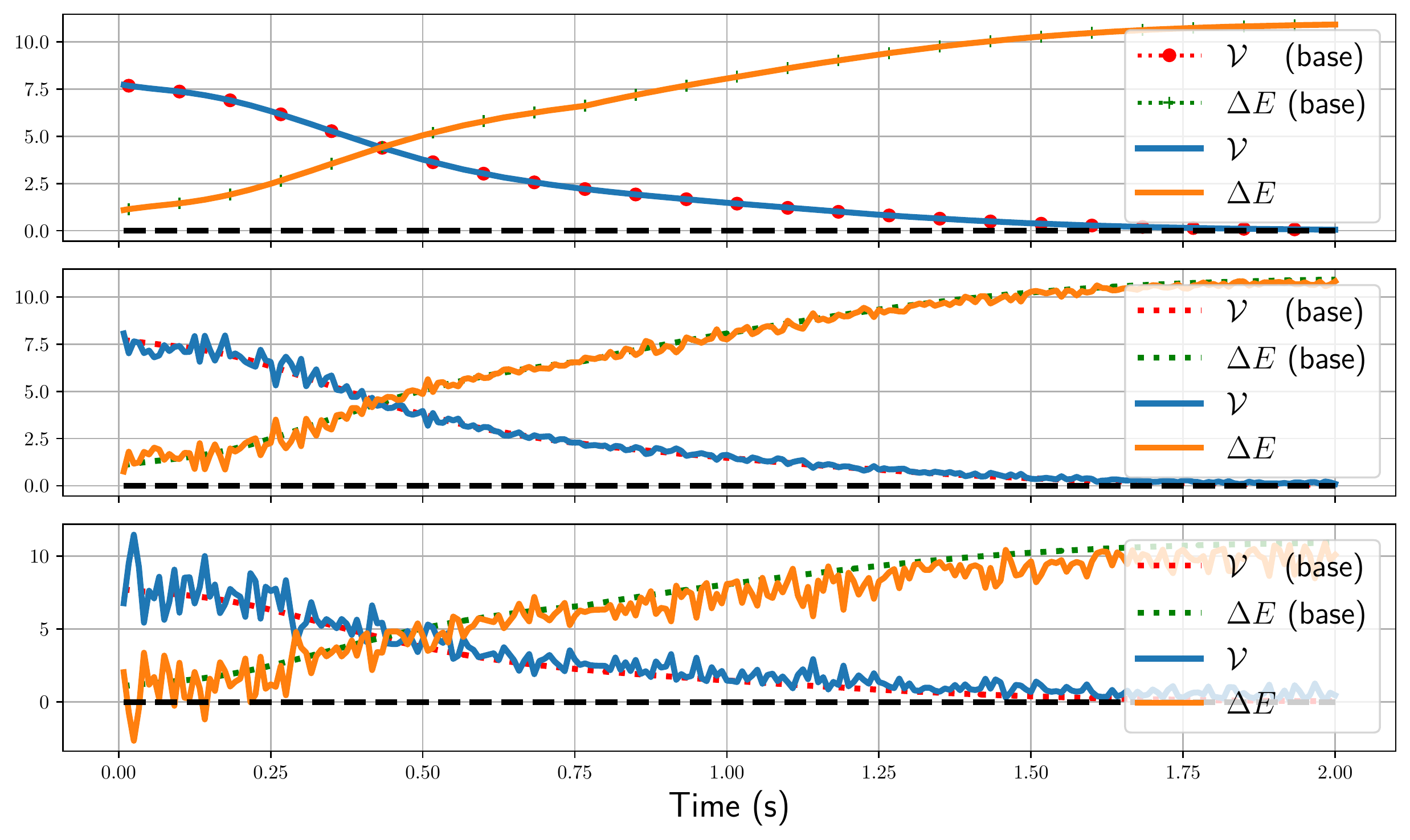}
\caption{\Revised{The Lyapunov function $\calV$ and safety margin $\Delta E$ in the presence of: high-frequency noise (top), state estimation error with standard deviation from \cite{hoang2021noise} (middle) and 3x larger (bottom), respectively).}}
\label{fig:robust_hex}
\end{figure}

%% file: tex/conclusion.tex
\section{Conclusion}
This paper developed a tracking controller for Hamiltonian systems with learned dynamics. We employed a neural ODE network to learn translation-invariant Hamiltonian dynamics on the $SE(3)$ manifold from trajectory data. The~Hamiltonian~of the learned system was used to synthesize an energy-shaping controller and quantify its robustness to modeling errors. A reference governor was employed to guide the system along a desired reference path using the trade-off between system energy, disturbance bounds, and distance to obstacles to guarantee safe tracking. Our results demonstrate that encoding $SE(3)$ kinematics and Hamiltonian dynamics in the model learning process achieves more accurate prediction than Gaussian Process regression. Utilizing the system energy in the control design offers a general approach for guaranteeing robustness and safety for physical systems and generalizes well to desired trajectories which are significantly different from the training data. Future work will focus on disturbance compensation and real experiments.



%% file: tex/appendix_robust_se3.tex
\section{Proof of Thm.~\ref{thm:robust_IDA_PBC}}
\label{app:robust_IDA_PBC_proof}

We do not write function arguments explicitly to simplify the notation. We also introduce the following notation for the components of $\bfe$ and $\frakp_e$ in \eqref{eq:error-dynamics}:
\begin{equation}\label{eq:error-components}
\begin{aligned}
\bfe &= \begin{bmatrix} \bfe_{\bfp} \\ \bfe_{\bfR}\end{bmatrix} = \begin{bmatrix} k_\bfp \bfR^\top\bfp_e \\ 
\frac{1}{2}k_{\bfR}\prl{\bfR_e - \bfR_e^\top}^{\vee}\end{bmatrix},\\
\frakp_e &= \bfM \begin{bmatrix} \bfe_{\bfv} \\ \bfe_{\bfomega} \end{bmatrix}= \bfM \begin{bmatrix} \bfv - \bfR_e^\top \bfv^* \\ \bfomega - \bfR_e^\top \bfomega^* \end{bmatrix}.
\end{aligned}
\end{equation}
Consider the Lyapunov function candidate in \eqref{eq:ISS_lyap}:
\begin{equation}
    \calV =  \frac{1}{2}\frakp_e^\top \bfM^{-1} \frakp_e + \calU_d + \rho \frac{d}{dt} \calU_d,
\end{equation}
where $\calU_d = \frac{k_{\bfp}}{2} \|\bfp_e\|^2 + \frac{k_{\bfR}}{2} \tr(\bfI-\bfR_e)$. In the domain $\calA$, we have \cite[Prop.~1]{taeyoung2013robust}:
\begin{equation} \label{eq:rot_err_bounds}
k_\bfR^{-2} \|\bfe_{\bfR}\|_2^2 \leq \tr(\bfI - \bfR_e) \leq \frac{4 k_\bfR^{-2}}{4-\alpha} \|\bfe_{\bfR}\|_2^2.
\end{equation}
By the chain rule and \eqref{eq:error-dynamics}, we have: 
\begin{equation}\label{eq:potential-dot}
\frac{d}{dt} \calU_d = \nabla_{\frakq_e}\!\calH_d^\top \dot{\frakq}_e = \nabla_{\frakq_e}\!\calH_d^\top\bfJ \bfM^{-1} \frakp_e = \bfe^\top\! \bfM^{-1} \frakp_e
\end{equation}
Using \eqref{eq:rot_err_bounds} and \eqref{eq:potential-dot}, together with the Cauchy-Schwartz inequality and the sub-multiplicative property of the Euclidean norm, the Lyapunov function candidate is bounded as:
\begin{align*}
	\calV &\leq \frac{\lambda_2}{2}  \norm{\frakp_e}^2 + \frac{k_\bfp^{-1}}{2}  \norm{\bfe_\bfp}^2 + \frac{2 k_\bfR^{-1}}{4 - \alpha} \norm{\bfe_\bfR}^2 + \rho \lambda_2 \norm{\bfe} \norm{\frakp_e}. \\
	\calV &\geq \frac{\lambda_1}{2}  \norm{\frakp_e}^2 + \frac{k_\bfp^{-1}}{2}  \norm{\bfe_\bfp}^2 + \frac{k_\bfR^{-1}}{2} \norm{\bfe_\bfR}^2 - \rho \lambda_2 \norm{\bfe} \norm{\frakp_e}.
\end{align*}
The bounds can be stated compactly in quadratic form using $\bfz = [\|\bfe\|\;\|\frakp_e\|]^\top$ and $\bfQ_1$, $\bfQ_2$ in \eqref{eq:Q-matrices}:
\begin{equation} \label{eq:bounds_on_lyap}
\frac{1}{2}\bfz^\top \bfQ_1 \bfz \leq  \calV \leq \frac{1}{2}\bfz^\top \bfQ_2 \bfz.
\end{equation}

The time derivative of the Lyapunov candidate satisfies:
\begin{equation*}
\frac{d}{dt}\calV = \frakp_e^\top \bfM^{-1} \dot{\frakp}_e +  \bfe^\top \bfM^{-1} \frakp_e + \rho \bfe^\top \bfM^{-1} \dot{\frakp}_e + \rho \dot{\bfe}^\top \bfM^{-1} \frakp_e.
\end{equation*}
The term $\dot{\frakp}_e$ is from \eqref{eq:error-dynamics}. The term $\dot{\bfe}$ is obtained from \eqref{eq:error-components}: 
\begin{equation}
\begin{aligned} \label{eq:edot}
\dot{\bfe} &=  \begin{bmatrix}
  \dot{\bfe}_{\bfp} \\ \dot{\bfe}_{\bfR}
  \end{bmatrix} = \begin{bmatrix} -\hat{\bfomega} \bfe_{\bfp} + k_{\bfp}\bfe_{\bfv}\\
k_{\bfR}\bfE_\bfR \bfe_{\bfomega}
\end{bmatrix}\\
&= -\begin{bmatrix} \hat{\bfomega} & \mathbf{0}\\ \mathbf{0}& \mathbf{0} \end{bmatrix} \bfe + \begin{bmatrix} k_{\bfp} \bfI & \mathbf{0}\\\mathbf{0} & k_{\bfR}\bfE_\bfR \end{bmatrix} \bfM^{-1} \frakp_e,
\end{aligned}
\end{equation}
where $\bfE_\bfR = \frac{1}{2}\brl{\tr(\bfR_e^\top) \bfI - \bfR_e^\top}$ satisfies $\|\bfE_\bfR\|_2 \leq 1$ \cite[Prop.~1]{lee2010geometric}. Hence, we have:
\begin{equation*}
\begin{aligned}
&\frac{d}{dt}\calV = -\frakp_e^\top\bfM^{-1}\bfK_{\bfd} \bfM^{-1} \frakp_e  + \frakp_e^\top \bfM^{-1} \bfd\\
&\quad- \rho\bfe^\top \bfM^{-1}\bfe  -\rho\bfe^\top \bfM^{-1}\bfK_{\bfd} \bfM^{-1} \frakp_e + \rho \bfe^\top \bfM^{-1} \bfd\\
&\quad- \rho  \frakp_e^\top \bfM^{-1} \begin{bmatrix} \hat{\bfomega} & \mathbf{0}\\ \mathbf{0}& \mathbf{0} \end{bmatrix} \bfe + \rho \frakp_e^\top \bfM^{-1} \begin{bmatrix} k_{\bfp} \bfI & \mathbf{0}\\\mathbf{0} & k_{\bfR}\bfE_\bfR \end{bmatrix} \bfM^{-1} \frakp_e.
\end{aligned}
\end{equation*}
To find an upper bound on $\frac{d}{dt}\calV$, we need a few intermediate steps. First, on the domain $\calA$, we have:
\begin{equation}\label{eq:omega_ub}
\left\lVert \begin{bmatrix} \hat{\bfomega} & \mathbf{0}\\ \mathbf{0}& \mathbf{0} \end{bmatrix}\right\rVert_2 = \norm{\hat{\bfomega}}_2 = \norm{\bfomega} \leq \| \bfM^{-1} \frakp\| \leq \lambda_2 \beta.
\end{equation}
Second, an upper bound on 
\begin{equation}
\xi_1 \coloneqq - \lambda_{\min}(\bfK_{\bfd}) \norm{\bfM^{-1} \frakp_e}^2  + \norm{\bfM^{-1} \frakp_e} \norm{\bfd} 
\end{equation}
can be found using Young's inequality \cite{ryalat2018robust}:
\begin{equation}
-\epsilon \norm{\bfa}^2  + \eta \norm{\bfa} \norm{\bfb} \leq -\frac{\epsilon}{2} \norm{\bfa}^2 + \frac{\eta^2}{2 \epsilon} \norm{\bfb}^2
\end{equation}
with $\epsilon = \lambda_{\min}(\bfK_d)$, $\eta=1$, $\bfa = \bfM^{-1} \frakp_e$, $\bfb = \bfd$:
\begin{equation}  \label{eq:xi_1ub}
\begin{aligned}
\xi_1 &\leq -\frac{\lambda_{\min}(\bfK_d)}{2} \norm{\bfM^{-1} \frakp_e}^2 
+ \frac{1}{2 \lambda_{\min}(\bfK_d)} \norm{\bfd}^2.
\end{aligned}
\end{equation}
Similarly, we have:
\begin{equation}\label{eq:xi_2ub}
\xi_2 \coloneqq - \lambda_1 \norm{\bfe}^2 +  \lambda_2 \norm{\bfe} \norm{\bfd} \leq - \frac{\lambda_1}{2}  \norm{\bfe}^2+ \frac{\lambda_2^2}{2 \lambda_1} \norm{\bfd}^2.
\end{equation}
Using \eqref{eq:omega_ub}, \eqref{eq:xi_1ub}, and \eqref{eq:xi_2ub}, $\frac{d}{dt}\calV$ is bounded by:
\begin{equation} \label{eq:bound_dotV}
\frac{d}{dt}\calV \leq - \frac{1}{2}\bfz^\top \bfQ_3 \bfz + k_\gamma \norm{\bfd}^2,
\end{equation}
where the elements of $\bfQ_3$ are provided in \eqref{eq:Q3-elements} and $k_\gamma = \frac{1}{2 \lambda_{\min}(\bfK_\bfd)} +   \frac{\rho \lambda_2^2}{2 \lambda_1}$. Since the parameters $\rho$, $k_{\bfp}$, $k_{\bfR}$, $\bfK_{\bfd}$ can be chosen arbitrarily, there exists some choice that ensures the matrices $\bfQ_1$, $\bfQ_2$, $\bfQ_3$ are positive definite as shown below. The inequalities in \eqref{eq:V_bounds} are obtained from \eqref{eq:bounds_on_lyap} and \eqref{eq:bound_dotV} using the Rayleigh-Ritz inequality.

\paragraph*{Region of Attraction}
We use the invariant sets $\calS_{c} = \crl{\bfx \mid \calV(\bfx, \bfx^*) \leq c}$ induced by the Lyapunov function to restrict the error dynamics inside the domain $\calA$ and estimate the region of attraction.

We determine $c_1 \geq 0$ such that $\dot{\calV}$ is positive on $\calS_{c_1}$. From \eqref{eq:V_bounds}, $\dot{\calV}$ is positive when $k_\gamma \delta_{\bfd}^2-k_3 \norm{\bfz}^2 \geq 0$, which happens when $\frac{\calV}{k_2} \leq \frac{k_\gamma}{k_3} \delta_{\bfd}^2$. Hence, with $c_1 = k_2 k_\gamma \delta_{\bfd}^2 / k_3$, we have $\dot{\calV} \geq 0$ on $\calS_{c_1}$. Then, we determine $c_2 \geq 0$ such that $\calS_{c_2} \subseteq \calA$. From \eqref{eq:rot_err_bounds} and \eqref{eq:bounds_on_lyap}, we have:
\begin{equation}
\frac{4-\alpha}{4} k_\bfR^2 \tr(\bfI - \bfR_e) \leq \norm{\bfe_{\bfR}}^2 \leq \norm{\bfz}^2 \leq \frac{\calV}{k_1}.
\end{equation}
Hence, if $\calV \leq \frac{1}{4} k_1 k_\bfR^2 \alpha \prl{4- \alpha}$, then $\tr(\bfI - \bfR_e) \leq \alpha$. Similarly, if $\calV \leq k_1 \beta^2$, then $ \norm{\frakp_e}^2 \leq \norm{\bfz}^2 \leq \frac{\calV}{k_1} \leq \beta^2$.
%
%
Hence, to ensure that $\calS_{c_2} \subseteq \calA$, we define $c_2$ as:
\begin{equation}
    c_2 \coloneqq k_1 \min \crl{k_\bfR^2 \alpha (4- \alpha)/4,  \beta^2}.
\end{equation}

To ensure that $c_1 < c_2$, the disturbance bound $\delta_\bfd$ must satisfy $\delta_\bfd < \sqrt{\frac{c_2k_3}{k_2 k_\gamma}}$. Then, any closed-loop system trajectory that starts in $\calS_{c_2}$ converges exponentially to $\calS_{c_1}$ and remains within it. Recall that $\bfe_\bfp = k_\bfp \bfR^\top \bfp_e$ and from \eqref{eq:bounds_on_lyap}:
\begin{equation}
k_{\bfp}^2 \norm{\bfp_e}^2 = \norm{\bfe_\bfp}^2 \leq \norm{\bfe}^2 \leq \norm{\bfz}^2 \leq \frac{\calV}{k_1}.
\end{equation}
Hence, on $\calS_{c_1}$, $\norm{\bfp_e}^2 \leq c_1 / (k_1 k_{\bfp}^2)$ and the uniform ultimate bound on the position error trajectory in \eqref{eq:uub} is satisfied.

\paragraph*{Design Parameter Choice}
We propose a systematic way to select parameters $\rho$, $k_\bfp$, $k_\bfR$, $\bfK_{\bfd}$, ensuring that the matrices $\bfQ_1$, $\bfQ_2$, $\bfQ_3$ are positive definite. Suppose $k_\bfp < \frac{4 - \alpha}{4} k_\bfR$ and $\bfK_\bfd = \gamma_d \bfI$ for some  $\gamma_d > 0$, then we have
\begin{equation}
\begin{aligned}
\bfQ_1 &= \begin{bmatrix} k_{\bfR}^{-1} & - \rho \lambda_2\\
- \rho \lambda_2 & \lambda_1 \end{bmatrix}
&\bfQ_2 &= \begin{bmatrix} k_\bfp^{-1} & \rho \lambda_2\\
\rho \lambda_2 & \lambda_2 \end{bmatrix}
\\
q_2 &= -\rho \lambda_2^2 \prl{\gamma_d + \beta} 
&q_3 &= \gamma_d \lambda_1^2  - 2\rho \lambda_2^2 k_\bfR .
\end{aligned}
\end{equation}
To guarantee the positive definiteness of $\bfQ_1$, $\bfQ_2$, $\bfQ_3$, the following requirements must be satisfied:
\begin{equation*} 
\begin{aligned}
\frac{\lambda_1}{k_\bfR} - \rho^2 \lambda_2^2 > 0,\; 
\frac{\lambda_2}{k_\bfp} - \rho^2 \lambda_2^2 > 0,\;
\gamma_d \lambda_1^2  - 2\rho \lambda_2^2 k_\bfR > 0 \\
\rho \lambda_1 \prl{\gamma_d \lambda_1^2  - 2\rho \lambda_2^2 k_\bfR} -  \rho^2 \lambda_2^4 \prl{\gamma_d + \beta}^2> 0.
\end{aligned}
\end{equation*}

All these constraints put upper bounds on $\rho$: 
\begin{equation}
\rho \leq \min \crl{\sqrt{\frac{\lambda_1}{k_\bfR \lambda_2^2}},\; \sqrt{\frac{1}{k_\bfp \lambda_2}},\; \frac{\gamma_d \lambda_1^2}{2 k_\bfR \lambda_2^2}, \; \bar{\rho}_{\bfQ_3}}, 
\end{equation}
where $\bar{\rho}_{\bfQ_3} = \frac{\gamma_d\lambda_1^3}{\lambda_2^2 \brl{2 \lambda_1 k_\bfR + \lambda_2^2 (\gamma_d + \beta)^2}}$.